\documentclass[twoside,11pt]{article}

\usepackage{thm-restate}
\usepackage{jmlr2e}

\firstpageno{1}

\title{Fully General Online Imitation Learning}

\author{\name Michael K.\ Cohen \email michael.cohen@eng.ox.ac.uk \\
       \addr Department of Engineering Science\\
       University of Oxford\\
       Future of Humanity Institute\\
       Oxford, UK OX1 3PJ
       \AND
       \name Marcus Hutter \email hutter1.net \\
       \addr DeepMind\\
       Department of Computer Science\\
       Australian National University\\
       Acton, ACT, Australia 2601
       \AND
       \name Neel Nanda \email neelnanda.io \\
       \addr Independent}
\editor{Joelle Pineau}

\usepackage[utf8]{inputenc}
\usepackage{amsmath,amsfonts,amssymb, bm}

\usepackage{url}
\usepackage{hyperref}
\usepackage[misc,clock,geometry]{ifsym}
\usepackage{enumitem}
\hypersetup{hidelinks}

\DeclareMathOperator*{\argmax}{argmax}
\DeclareMathOperator*{\KL}{KL}
\DeclareMathOperator{\p}{P}
\DeclareMathOperator{\E}{\mathbb{E}}
\DeclareMathOperator*{\lequal}{\leq}

\DeclareMathOperator*{\gequal}{\geq}
\DeclareMathOperator*{\equal}{=}

\DeclareMathOperator*{\va}{\!\bigm\vert\!}
\DeclareMathOperator*{\vb}{\!\Bigm\vert\!}
\DeclareMathOperator*{\vc}{\!\biggm\vert\!}
\DeclareMathOperator*{\vd}{\!\Biggm\vert\!}
\DeclareMathOperator{\A}{\mathcal{A}}
\DeclareMathOperator{\Ob}{\mathcal{O}}

\DeclareMathOperator{\M}{\mathcal{M}}

\DeclareMathOperator{\ptrue}{P^{\pi^\mathnormal{i}_\alpha}_\mu}
\DeclareMathOperator{\rhosn}{\rho^{stat}_{\mathnormal{n}}}

\DeclareMathOperator{\rhosnm}{\rho^{stat}_{\mathnormal{n -- }1}}
\DeclareMathOperator{\rhoso}{\rho^{stat}_{1}}
\DeclareMathOperator{\rhonn}{\rho^{norm}_{\mathnormal{n}}}
\DeclareMathOperator{\Etrue}{\mathbb{E}^{\pi^\mathnormal{i}_\alpha}_\mu}
\DeclareMathOperator{\pa}{\Pi^\alpha_{\mathnormal{h_{<t}}}}

\DeclareMathOperator*{\bigtimes}{\times}

\newcommand{\tagaligneq}{\refstepcounter{equation}\tag{\theequation}}

\newtheorem{assumption}{Assumption}

\usepackage{tikz}
\def\partialbox{
    \tikz\draw[path picture={\fill[black] (path picture bounding box.north east)
  -- (path picture bounding box.south west) |-cycle;}] (0,0) rectangle  ++ (0.25,0.25);
}

\makeatletter

\newcommand{\proofoutlinename}{Proof -- Detailed Outline}
\makeatother

\usepackage{tikz}
\def\dottedbox{\tikz\node[draw=black,dotted] {\phantom{}};}

\makeatletter
\newenvironment{proofidea}[1][\proofideaname]{\par
  \normalfont
  \topsep6\p@\@plus6\p@ \trivlist
    \item[\,\textbf{
    #1}]
}{%
  \hfill$\dottedbox$ \endtrivlist
}
\newcommand{\proofideaname}{Proof idea}
\makeatother

\usepackage{natbib}
\usepackage{etoolbox}
\makeatletter
\patchcmd{\NAT@test}{\else \NAT@nm}{\else \NAT@nmfmt{\NAT@nm}}{}{}
\DeclareRobustCommand\citepos
  {\begingroup
  \let\NAT@nmfmt\NAT@posfmt
  \NAT@swafalse\let\NAT@ctype\z@\NAT@partrue
  \@ifstar{\NAT@fulltrue\NAT@citetp}{\NAT@fullfalse\NAT@citetp}}
\let\NAT@orig@nmfmt\NAT@nmfmt
\def\NAT@posfmt#1{\NAT@orig@nmfmt{#1's}}
\makeatother

\usepackage{lastpage}
\jmlrheading{23}{2022}{1-\pageref{LastPage}}{6/21}{9/22}{21-0618}{Michael K. Cohen, Marcus Hutter, and Neel Nanda}
\ShortHeadings{Fully General Online Imitation Learning}{Cohen, Hutter, and Nanda}

\begin{document}

\allowdisplaybreaks

\maketitle

\begin{abstract}%
In imitation learning, imitators and demonstrators are policies for picking actions given past interactions with the environment. If we run an imitator, we probably want events to unfold similarly to the way they would have if the demonstrator had been acting the whole time. In general, one mistake during learning can lead to completely different events. In the special setting of environments that restart, existing work provides formal guidance in how to imitate so that events unfold similarly, but outside that setting, no formal guidance exists.
We address a fully general setting, in which the (stochastic) environment and demonstrator never reset, not even for training purposes, and we allow our imitator to learn online from the demonstrator. Our new conservative Bayesian imitation learner underestimates the probabilities of each available action, and queries for more data with the remaining probability. Our main result: if an event would have been unlikely had the demonstrator acted the whole time, that event's likelihood can be bounded above when running the (initially totally ignorant) imitator instead. Meanwhile, queries to the demonstrator rapidly diminish in frequency. If any such event qualifies as ``dangerous'', our imitator would have the notable distinction of being relatively ``safe''.
\end{abstract}


\section{Introduction}

Supervised learning of independent and identically distributed data is often practiced in two phases: training and deployment. This separation makes less sense if the learner's predictions affect the distribution of future contexts for prediction, since the deployment phase could lose all resemblance to the training phase. When a program's output changes its future percepts, we often call its output ``actions''. Supervised learning in that regime is commonly called ``imitation learning'', where labels are the actions of a ``demonstrator'' \citep{syed2010reduction}. Our agent, acting in a general environment that responds to its actions, tries to pick actions according to the same distribution as a demonstrator.

Even in imitation learning, where it is understood that actions can change the distribution of contexts that the agent will face, it is common to separate a training phase from a deployment phase. This assumes away the possibility that the distribution of contexts will shift significantly upon deployment and render the training data increasingly irrelevant. Here, we present an online imitation learner that is robust to this possibility.

The obvious downside is that the training never ends. The agent can always make queries for more data, but importantly, it does this with diminishing probability. It transitions smoothly from a mostly-training phase to a mostly-deployed phase. Our agent also handles totally general stochastic environments (environments serve new contexts for the agent to act in) and totally general stochastic demonstrator policies. No finite-state-Markov-style stationarity assumption is required for either. The lack of assumptions about the environment is a mundane point, because imitation learners don't have to learn the dynamics of the environment, but the lack of assumptions on the prediction target---the demonstrator's policy---makes these results highly non-trivial. The only assumption is that the demonstrator's policy belongs to some known countable class of possibilities. Moreover, stochasticity makes single-elimination-style learning \citep{gold1967language} impossible.

For demonstrator policies this general, we present formal results that are unthinkable in the train-then-deploy paradigm. The $\ell_1$ distance between the imitator and demonstrator policies converges to 0 in mean cube, when conditioned on a high-probability event (Theorem \ref{thm:imtodem}). And Theorem \ref{thm:truthintop} shows that the event has high probability. Conditioned on the same high-probability event, we bound the KL divergence from imitator to demonstrator (Theorem \ref{thm:kl}), and we upper bound the probability of an arbitrary event under the imitator's policy, given a low probability of occurrence under the demonstrator's policy (Theorem \ref{thm:badevent}).
Instead of having a finite training phase, our agent's query probability converges to 0 in mean cube (Theorem \ref{thm:queryprob}). Without Theorems \ref{thm:queryprob} and \ref{thm:truthintop}, the remaining theorems would be uninteresting; they would be easily fulfilled by an imitator that always queried the demonstrator, or they would apply only rarely.

Our imitator maintains a posterior over demonstrator models. At each timestep, it takes the top few demonstrator models in the posterior, in a way that depends on a scalar parameter $\alpha$. Then, for each action, it considers the minimum over those models of the probability that the demonstrator picks that action. The imitator samples an action according to those probabilities, and if no action is sampled (since model disagreement makes the probabilities to sum to less than 1), it defers to the demonstrator.

We review theoretical developments in imitation learning in Section \ref{sec:related}, define our formal setting in Section \ref{sec:prelim}, define our imitation learner in Section \ref{sec:imitation}, and illustrate it with a toy example in Section \ref{sec:example}. We state key formal results in Section \ref{sec:results}, and we outline our proof technique and introduce necessary notation in Section \ref{sec:proofnotation}. Section \ref{sec:resultsseq} presents lemmas and intermediate results, and Section \ref{sec:keyproofs} presents proofs and proof ideas of our key results, but most of the proofs appear in Appendix \ref{sec:proofs}. Appendix \ref{sec:notationtable} collects notation and definitions.

\section{Related Work} \label{sec:related}
Recall that a key difficulty of imitation learning over supervised learning is the removal of a standard i.i.d.\ assumption. However, all existing formal work in imitation learning studies repeated finite episodes of length $T$; even though the dynamics are not i.i.d.\ from timestep to timestep within an episode, the agent learns from a sequence of episodes that are, as a whole, independent and identically distributed. Thus, the scope of existing formal work is limited to environments that ``restart''. A driving agent that gets housed in a new car every time it crashes (or gets hopelessly lost) enjoys a ``restarting'' environment, whereas a driving agent with only one car to burn does not. If we can accurately simulate a non-restarting environment, then training the imitator in simulation (using existing formal methods) could indeed prepare it to act in a non-restarting one. The viability of this approach depends on the environment; for many, we simply cannot simulate them with enough accuracy. For example, consider imitating a sales rep at a software company, interfacing with potential clients over email. For a real potential client, a relationship cannot be rebooted, and no simulation could anticipate the many diverse needs of clients.

In the context of restarting environments, \citet{syed2010reduction} reduce the problem of predicting a demonstrator's behavior to i.i.d.\ classification. The only assumption about the demonstrator is that the value of its policy as a function of state is arbitrarily well approximated by the value of a deterministic policy, which is only slightly weaker than assuming the demonstrator is deterministic itself. They make no assumptions about the environment, other than that we can access identical copies of it repeatedly. They show that if a classifier guessing the demonstrator's actions has an error rate of $\varepsilon$, then the value of the imitator's policy that uses the classifier is within $O(\sqrt{\varepsilon})$ of the demonstrator.

\citet{judah2014active} improve the label complexity of \citepos{syed2010reduction} reduction by actively deciding when to query the demonstrator, instead of simply observing $N$ full episodes before acting. Making the same assumptions as that paper, and also assuming a realizable hypothesis class with a finite VC dimension, they attempt to reduce the number of queries before the agent can act for a whole episode on its own with an error rate less than $\varepsilon$. Letting $T$ be the length of an episode, compared to \citepos{syed2010reduction} $O(T^3/\varepsilon)$ labels, they achieve $O(T \log(T^3 / \varepsilon))$.

\citet{ross2010efficient} also reduce the problem to classification. In a trivial reduction, the imitator observes the demonstrator act from the distribution of states induced by the demonstrator policy. In this reduction, if the classifier has an error rate of $\varepsilon$ per action on the demonstrator's state distribution, the error rate of the imitator on its own distribution is at most $T^2 \varepsilon$, where $T$ is again the length of the episode. Their main contribution is to introduce a cleverer training regime for the classifier to reduce this bound to $T \varepsilon$ in environments with approximate recoverability.

\citet{ross2011reduction} reduce imitation learning to something else: a no-regret online learner, for which the average error rate over its lifetime approaches 0, even with a potentially changing loss function. With access to an online learner with average regret $O(1/N_{\textrm{predictions}})$, they construct an imitation learner with regret of the same order. Unlike \citet{syed2010reduction} and \citet{judah2014active}, they make no assumption that the demonstrator is arbitrarily well-approximated by a deterministic policy. Unlike \citet{judah2014active}, they do not assume a realizable hypothesis class with a finite VC dimension. And unlike  the \citet{ross2010efficient} (for their main contribution), they do not assume approximate recoverability. They do still assume that we can repeatedly access identical copies of the environment, and the loss function used for their measurement of regret must be bounded. To achieve a regret of order $O(1/N_{\textrm{predictions}})$ with probability at least $1-\delta$, they require $O(T^2 \log (1/\delta))$ observations of the demonstrator.

There is a great deal of empirical study of imitation learning, given the practical applications, which \citet{hussein2017imitation} review. We call a few specific experiments to the reader's attention, since they resemble our work in taking an active approach to querying, with an eye to risk aversion, not just label efficiency; they find it works. First, \citet{brown2018risk,brown2020bayesian} consider a context where the imitator can, at any time, ask the demonstrator how it would act in any of finitely many states. These imitators focus on states that they assign higher value at risk. Those papers and the following all show strong label efficiency alongside limited loss. \citet{zhang2017query} assume some method of predicting the error of an imitator in the process of learning, and they query for help when it is above some threshold. Otherwise, their imitator follows \citepos{ross2011reduction} construction. In their paper, the function that predicts the imitator's error is learned from hand-picked features of a dataset. \citet{menda2019ensembledagger} query much more extensively, but like \citet{zhang2017query}, they don't always act on the demonstrator's suggestion, in order to sample a more diverse set of states. Unlike \citet{zhang2017query}, they \textit{do} act on it when the imitator's action deviates enough from the demonstrator's (given some hand-designed distance metric over the action space). They also defer to the demonstrator when there is sufficient disagreement among an ensemble of imitators. They find their imitator is more robust. \citet{hoque2021lazydagger} note that in many contexts, it is more convenient for the demonstrator to be queried a few times successively, rather than spread out over a long time. They modify \citepos{zhang2017query} approach: the imitator starts querying when the estimated error exceeds the same threshold, but it continues querying until it returns below a lower threshold. At the cost of more total queries, it requires fewer query-periods. Like the formal work, all these experiments regard environments that restart.

Adjacent to pure imitation learning (trying to pick the same actions as a demonstrator would), there is also work on trying to act in pursuit of the same goals as a demonstrator (which must be inferred), or matching only some outcomes of the demonstrator policy, like the expectation of some given set of features. For a review of some work in this area, see \citet{adams2022survey}.

\section{Preliminaries} \label{sec:prelim}

Let $a_t \in \A$ and $o_t \in \Ob$ be the action and observation at timestep $t \in \mathbb{N}$. Let $q_t \in \{0, 1\}$ denote whether the imitator ($q_t=0$) or demonstrator ($q_t=1$) selects $a_t$. Let $\mathcal{H} = \{0, 1\} \times \A \times \Ob$, and let $h_t = (q_t, a_t, o_t) \in \mathcal{H}$. Let $h_{<t} = (h_0, h_1, ..., h_{t-1})$. $\mathcal{X}^n = \bigtimes_{i = 1}^n \mathcal{X}$ denotes the set of $n$-tuples of elements of $\mathcal{X}$, and $\mathcal{X}^* = \bigcup_{n = 0}^\infty \mathcal{X}^n$ is the Kleene-star operator, which denotes all tuples of elements of $\mathcal{X}$.

Let $\pi : \mathcal{H}^* \rightsquigarrow \{0, 1\} \times \A$, and $\rightsquigarrow$ denotes that $\pi$ gives a distribution over $\{0, 1\} \times \A$. $\epsilon$ will denote the empty string; it is the element of $\mathcal{H}^0$. $\pi$ is called a policy, and will typically be written $\pi(q_t a_t \mid h_{<t})$. $\pi(a_t \mid h_{<t})$ denotes the marginal distribution over the action. Let $\mu : \mathcal{H}^* \times \{0, 1\} \times \A \rightsquigarrow \mathcal{O}$. $\mu$ is called the environment, and will typically be written $\mu(o_t \mid h_{<t} q_t a_t)$. Note from this construction that an environment and a policy may qualitatively change over time---instead of being stationary with respect to the latest timestep, they can depend on the whole history.

Much formal work in imitation learning and reinforcement learning involves defining environments in terms of their Markov states and how one transitions through them. The defining property of a state is that that future is independent of the past conditioned on the state. For those more comfortable in that framework, our state space here is $\mathcal{H}^*$, so the Markov property trivial: the state is the whole history, so indeed, the future is independent of the history, when conditioned on the history. The point of the Markov Decision Process formalism is that when the state space is finite (or compact, with relevant functions of it being continuous), more tractable inference algorithms become available, but we do not assume finiteness or any structure in the state space. For finite histories denoted $h_{<t}$, the reader could mentally substitute $s_t$, this being the state at time $t$, but the infinite history $h_{<\infty}$, which appears in some proofs, has no standard notational analog.

Speaking of which, let $\mathcal{H}^\infty$ be the set of infinite strings of elements of $\mathcal{H}$. Let $\p^\pi_\mu$ be the probability measure over $\mathcal{H}^\infty$ where query records and actions are sampled from $\pi$, and observations are sampled from $\mu$. The event space is the standard sigma algebra over cylinder sets $\sigma(\{\{h_{<t}h_{t:\infty} : h_{t:\infty} \in \mathcal{H}^\infty\} : h_{<t} \in \mathcal{H}^*\})$. In a stochastic process, a cylinder set is the set of all possible futures given a particular past.

Let $\Pi$ be a finite or countable set of policies, and for $\pi \in \Pi$, let $w(\pi) > 0$ be a prior weight assigned to $\pi$, such that $\sum_{\pi \in \Pi} w(\pi) = 1$. This represents the imitator's initial belief distribution over the demonstrator's policy. For convenience, let $\Pi$ only contain policies which assign zero probability to $q_t = 0$, since demonstrator models may as well be convinced that the demonstrator is picking the action.
\begin{example}[(Linear-Time) Computable Policies] \label{ex:policies}
The requirement that $\Pi$ be countable is not restrictive in theory. Suppose $\Pi$ is the set of programs that compute a policy (in linear time). These can be easily enumerated, and the prior $w$ can be set $\propto 2^{-\textrm{program length}}$ \citep{kraft1949device,Hutter:04uaibook}.
\end{example}
Given the near absence of constraints, the choice of model class might pique philosophical interest. There are multiple logics with differing powers that we could plausibly use to represent programs, including ``programs'' higher in the arithmetic hierarchy. In general, the choice of programming language would change programs' relative length, and there are no clear desiderata when choosing a language. So Example \ref{ex:policies} does not appear to offer an approach to solving the Problem of Priors \citep{talbottepistemologybayesian}. The option to restrict to linear-time programs is a marginally more practical possibility that might escape most philosophical discussions.

\section{Imitation} \label{sec:imitation}

Let $w(\pi \mid h_{<t})$ be the posterior weight after observing $h_{<t}$ that demonstrator-chosen actions were sampled from $\pi$. That is,

\begin{equation} \label{eqn:post}
    w(\pi \mid h_{<t}) : \propto w(\pi) \prod_{k < t : q_k = 1} \pi(q_k a_k \mid h_{<k})
\end{equation}
normalized such that $\sum_{\pi \in \Pi} w(\pi \mid h_{<t}) = 1$. Ranking the policies by posterior weight, let $\pi^{h_{<t}}_n$ be the one with the $n$\textsuperscript{th} largest posterior weight $w(\pi \mid h_{<t})$, breaking ties arbitrarily. Now let $\pa$ be the set of policies with posterior weights at least $\alpha$ times the sum of the posterior weights of policies that are at least as likely as it; that is,
\begin{equation} \label{eqn:topmodels}
    \pa := \{\pi^{h_{<t}}_n \in \Pi : w(\pi^{h_{<t}}_n \mid h_{<t}) \geq \alpha \sum_{m \leq n} w(\pi^{h_{<t}}_m \mid h_{<t})\}
\end{equation}
This is the set of policies the imitator takes seriously. The imitator is designed to be robust to policies in this set, so smaller $\alpha$ will make it more robust.
Let $\pi^d$ denote the demonstrator's policy, defined such that $\pi^d(q_t = 1 \mid h_{<t}) = 1$ for all values of $h_{<t}$. As later results suggest, $\alpha$ should be set a few orders of magnitude below $w(\pi^d)$; since $\pi^d$ is probably unknown to the programmers, or else there would be no need for imitation learning, $w(\pi^d)$ will have to be estimated. The imitator's policy $\pi^i_\alpha$ is defined in the next two equations:
\begin{equation} \label{eqn:imitator}
    \pi^i_\alpha(0, a \mid h_{<t}) := \min_{\pi' \in \pa} \pi'(1, a \mid h_{<t})
\end{equation}
The $0$ on the l.h.s. means the imitator is picking the action itself instead of deferring to the demonstrator, and the 1 on the r.h.s. means this is the probability of the demonstrator model $\pi'$ picking that same action.

The imitator uses the leftover probability to query. Let $\theta_q(h_{<t}) := 1 - \sum_{a \in \A} \pi^i_\alpha(0, a \mid h_{<t})$. $\theta_q$ is the probability with which the imitator queries the demonstrator to have it pick the action. Thus,
\begin{equation}
    \pi^i_\alpha(1, a \mid h_{<t}) := \theta_q(h_{<t}) \pi^d(1, a \mid h_{<t})
\end{equation}
One can see that $q_t$ records whether the demonstrator was involved in selecting the action. Using the model class and prior from Example \ref{ex:policies}, the time-complexity constraint makes $\pi^i_\alpha$ computable.

Conservatism with respect to probability estimates is a core technical innovation of our work. Taking the minimum over a set of models with high posterior weights is an approach to conservatism inspired by \citepos{cohen2020pessimism} pessimistic agent. The pessimistic agent, unlike ours, is a reinforcement learner, but it is also designed to keep certain (risky) events unlikely. By underestimating probabilities, the imitator only acts if it is sure the demonstrator might act that way.

We will also consider hypothetical imitator policies if the demonstrator policy were something else; for an arbitrary demonstrator policy $\pi$, let $\hat{\pi}_\alpha$ denote the corresponding imitator policy, so $\pi^i_\alpha = \hat{(\pi^d)}_\alpha$. This paper will investigate the probability distribution $\ptrue$ and compare it to $\p^{\pi^d}_\mu$.

\section{Toy Example} \label{sec:example}

We now walk through a toy example, in which our imitation learner has about a half-million demonstrator models in its model class $\Pi$. We begin by defining $\Pi$. The action space $\mathcal{A}$ of the demonstrator is $\texttt{null} \cup \{0, 1\}^4$. The observation space $\mathcal{O}$ is $\{\textrm{``''}, 1, 2, 3\}$. A demonstrator model $\pi \in \Pi$ defined by is a 12-tuple of the elements $\{1/3, 2/3, 1\}$. When the latest observation is 1, 2, or 3, let $x$ be the 1\textsuperscript{st} - 4\textsuperscript{th}, 5\textsuperscript{th} - 8\textsuperscript{th}, or 9\textsuperscript{th} - 12\textsuperscript{th} elements of 12-tuple. Then, the demonstrator model outputs four bits that are Bernoulli distributed according to each of the four elements of $x$. All demonstrator models output \texttt{null} when the latest observation is ``''. The true demonstrator also takes the form of such a demonstrator model. Each observation is randomly sampled; it is 1 with probability $1/4$, 2 with probability $1/16$, 3 with probability $1/64$, and otherwise ``''.

Let's give some flavor to this example. The demonstrator does client relations for a high-end travel agency with very fussy clients. The demonstrator gets a feel for her clients, and for any given night that a client needs a restaurant recommendation, the demonstrator sends a Boolean 4-tuple to the restaurant team, who identifies a suitable restaurant. The observation tells the demonstrator which of the three clients needs a recommendation, if any. The first bit of the Boolean 4-tuple tells the restaurant team whether the restaurant should have lots of vegetarian options, the second bit: should it have a Michelin star, the third: should it have unfamiliar local specialties, and the fourth: should it be Instagrammable. Why is the demonstrator stochastic? Many clients want a variety of styles of restaurants from night to night. The demonstrator couldn't write down the exact probabilities that she is using to generate these Boolean vectors; she goes off of intuition. If we run an imitator that only sometimes asks the demonstrator for help, we can free up some of the demonstrator’s time.

Unfortunately, in this toy environment, the fussy clients sometimes quit. Each client has a 4-tuple of probabilities that they would like their Boolean vector sampled from (conveniently in $\{1/3, 2/3, 1\}^4$). If it becomes clear that this is not how their Boolean vectors are being sampled, they quit. (``Becoming clear'' is operationalized as follows: $H_1$ is the hypothesis that their restaurant recommendations are being sampled correctly; $H_2$ is the hypothesis that some other 4-tuple in $\{1/3, 2/3, 1\}^4$ is producing their restaurant recommendations. If, given the set of all restaurant recommendations they have gotten, the likelihood ratio of $H_2$ exceeds 100, the client quits. Note that this happens if an element is ever \texttt{False} when it was supposed to be \texttt{True} with probability 1; some clients demand Michelin stars.) When recommendations are made by the demonstrator, who always correctly intuits the client’s desired distribution of restaurants, clients hardly ever quit. We would like clients to hardly ever quit even when the imitator frequently takes over.

For an imitator with $\alpha=1\textrm{e-14}$, Figure \ref{fig:queryrecord} shows how often it has to query the demonstrator to pick the restaurant features. Recommendations are random, and this is only one run. Running it with 20 different random seeds, the number of queries required is $486.75 \pm 52.63$ (out of $2^{15}$ timesteps), and no client ever quit. Returning to run depicted in Figure \ref{fig:queryrecord}, Table \ref{tab:exampleposterior} works through an example of the posterior and the imitator's behavior. The code for this toy example can be found at \url{https://tinyurl.com/imitation-toy-example}.

\begin{figure}
    \centering
    \includegraphics[width=\linewidth]{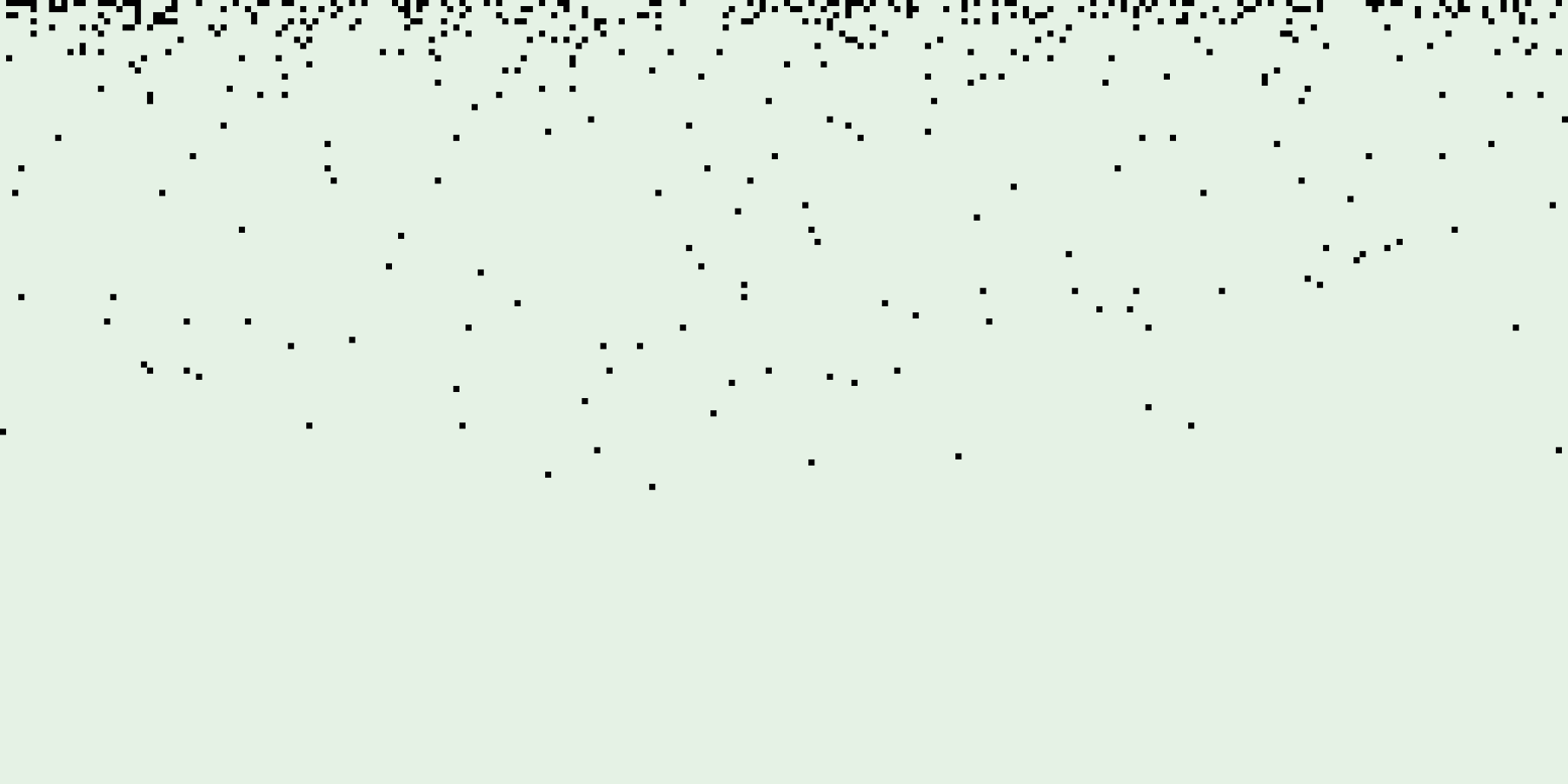}
    \caption{Timesteps when the imitator queries. $2^{15}$ timesteps are shown, with black representing a query, and green representing the imitator acting unassisted. Pixels are to be read like text, left to right, top to bottom. In the accompanying code, a random seed of 0 is used to generate this image.}
    \label{fig:queryrecord}
\end{figure}

\begin{table}[h]
    \centering
    \footnotesize
    \hfill
    \begin{tabular}{ccc}
1/3 & 2/3 & 1 \\
    \hline
\textbf{ -0.0000} & -44.0000 &     -inf \\
-50.0000 & \textbf{ -0.0000} &     -inf \\
\textbf{  0.0000} & -78.0000 &     -inf \\
-46.0000 & \textbf{ -0.0000} &     -inf \\
\\
-18.0000 & \textbf{ -0.0000} &     -inf \\
-69.7384 & -25.7384 & \textbf{ -0.0000} \\
\textbf{ -0.0000} & -22.0000 &     -inf \\
-69.7384 & -25.7384 & \textbf{ -0.0000} \\
\\
\textbf{ -0.3219} &  -2.3219 &     -inf \\
\textbf{ -0.0056} &  -8.0056 &     -inf \\
\textbf{ -0.0000} & -16.0000 &     -inf \\
-28.5303 & -10.5303 & \textbf{ -0.0010}
    \end{tabular}
    \hfill
\begin{tabular}{cc}
    $p([\texttt{False},  \texttt{True}, \texttt{False},  \texttt{True}] \mid \textrm{client 2})$ & Model  \\
    \hline
    0.11111 & (..., 2/3, 1, 2/3, 1, ...) \\
    0.14815 & (..., 2/3, 1/3, 1/3, 1, ...) \\
    0.22222 & \textbf{(..., 2/3, 1, 1/3, 1, ...)} \\
    0.29630 & (..., 1/3, 2/3, 1/3, 1, ...) \\
    0.44444 & (..., 1/3, 1, 1/3, 1, ...)
\end{tabular}
    \hfill
    \caption{\textit{Left:} Log$_2$ posterior at timestep 1000 for the run depicted in Figure \ref{fig:queryrecord}. The posterior decomposes into posterior probabilities for each of 12 features. Each block is a client, each row is a feature, and each entry is the log posterior probability that the demonstrator picks \texttt{True} for that feature with probability 1/3, 2/3, or 1, respectively. To get the posterior for a whole demonstrator model, as in Equation \ref{eqn:post}, add the independent posteriors for each element in the 12-tuple of the demonstrator model. The posterior weight on the truth is in bold for each feature; that is, the true demonstrator for this run is (1/3, 2/3, 1/3, 2/3, 2/3, 1, 1/3, 1, 1/3, 1/3, 1/3, 1). 
    \textit{Right:} At timestep 1000, with $\alpha=1\textrm{e-14}$, we have many top models, as defined in Equation \ref{eqn:topmodels}. The first column is a list of probabilities that different top models assign to the outcome [False True False True] for client 2. The second column contains examples of top models that assign those probabilities to the outcome [False True False True] for client 2, with the true model in bold. Recall a demonstrator model is defined by a 12-tuple, but the only relevant elements for client 2 are 5-8. All these models have posterior weight large enough to make it into the top set. Thus, the probability the imitator picks [\texttt{False}  \texttt{True} \texttt{False}  \texttt{True}] for client 2 is 0.11111, the minimum probability shown, as per Equation \ref{eqn:imitator}.
    }
    \label{tab:exampleposterior}
\end{table}

\section{Results}\label{sec:results}
For the whole of the paper, we assume:
\begin{assumption}[Realizability]
$\pi^d \in \Pi$.
\end{assumption}
That is, the imitator can conceive of the demonstrator. There may be some interesting results in the setting of approximate realizability, where $\exists \pi \in \Pi$ such that $\pi \approx_\varepsilon \pi^d$ in some sense, but that is out of our scope here.

We now state and discuss our key results before turning to selected proofs. Our first is that the imitator's query probability converges to 0 in mean cube. This result
\begin{itemize}
    \itemsep0em
    \item renders its resemblance to the demonstrator non-trivial, since always querying would yield perfect correspondence,
    \item is desirable in its own right if demonstrator access is a limited resource,
    \item and is instrumental in proving the remaining results, since low query probability implies little model disagreement.
\end{itemize}

\begin{restatable}[Limited Querying]{theorem}{thmqueryprob} \label{thm:queryprob}
\begin{equation*}
    \Etrue \left[\sum_{t=0}^\infty \theta_q(h_{<t})^3 \right] \leq |\mathcal{A}| \alpha^{-3}(24  w(\pi^d)^{-1} + 12)
\end{equation*}
\end{restatable}

The in mean cube bound allows infinite querying, but it diminishes in frequency, or else the expectation of an infinite sum of cubed probabilities would not be finite. Since we query under uncertainty, both querying and uncertainty diminish in tandem; this is a theme for active learners in general. Error bounds in Bayesian prediction and MAP prediction tend to be $\Theta(\log(w(\textrm{truth})^{-1}))$ and $\Theta(w(\textrm{truth})^{-1})$ respectively, so theoretically, our case resembles the MAP one. The cubic dependence on $\alpha$ is unfortunate, and subsequent results inherit them; the only path we found to proving a bound was fairly circuitous, and we are unsure whether this dependence can be improved.

Our remaining results show that the imitator resembles the demonstrator on one condition: $\pi^d \in \pa$. Recall that $\pa$ is a set of top demonstrator models that the imitator takes seriously, and $\pi^d$ is the true demonstrator model. Low model disagreement implies high accuracy when the truth is one of those models, and recall that our querying regime promises low model disagreement within finite time.

Fortunately, this condition has high probability for $\alpha <\!\!< w(\pi^d)$.
\begin{restatable}[Top Models Contain Truth]{theorem}{thmtruthintop}
\label{thm:truthintop}
$\ptrue(\forall t : \pi^d \in \pa) \geq 1 - \alpha w(\pi^d)^{-1}$
\end{restatable}
Let $E$ be the event $\forall t : \pi^d \in \pa$, so the true demonstrator policy is always in the top set. The high probability of $E$ is mainly of interest in the context of subsequent results that depend on it. For instance, conditioned on $E$, the imitator, when picking its own actions, converges to the demonstrator in mean cube.
\begin{restatable}[Predictive Convergence]{theorem}{thmimtodem} \label{thm:imtodem}
For $\alpha < w(\pi^d)$,
\begin{equation*}
    \Etrue \left[\sum_{t=0}^\infty \left(\sum_{a \in \A} \va \pi^i_\alpha(0, a \mid h_{<t}) - \pi^d(1, a \mid h_{<t}) \va \right)^3 \vd E \right] \leq \frac{|\mathcal{A}| \alpha^{-3}(24  w(\pi^d)^{-1} + 12)}{1 - \alpha w(\pi^d)^{-1}}
\end{equation*}
\end{restatable}
This theorem finally justifies our calling $\pi^i_\alpha$ an ``imitator'', since the policy converges to that of the demonstrator. Existing literature on imitation learning does little to suggest that imitators exist in non-restarting environments. This result shows that they do, at least in a high-probability sense. Note that the denominator is the probability of $E$, which will be nearly $1$ for appropriate choice of $\alpha$. The requirement that $\alpha < w(\pi^d)$ has important consequence: when $\alpha$ is set appropriately, the bounds in this theorem and Theorem \ref{thm:queryprob} are effectively quartic in $w(\pi^d)^{-1}$. We do not know if a better rate is possible under additional assumptions. It is even possible that stronger results are available without additional assumptions, and we simply failed to identify them. We think this is a ripe area for research.

We argue informally that this disappointing dependence can be mitigated in some circumstances. By pre-training with $N$ consecutive demonstrator queries and calling the posterior at that point the new ``prior'' for the purposes of our analysis, the ``prior'' on $w(\pi^d)$ could usually be made quite large, unless most demonstrator models behave extremely similarly for the first $N$ steps. Consider an extreme case: many models of comparable weight almost agree with the true model, except one disagrees at $t=1$, one at $t=2$, etc. In this case, the posterior on the truth increases very slightly every step, as models are excluded one by one. If, on the other hand, half of demonstrator models confidently predict one action, and half confidently predict another, the posterior on the truth will likely nearly double in one step. So to the extent that a large fraction of models in $\Pi$ disagree with $\pi^d$ within the first $N$ steps, the posterior on the truth would increase exponentially following pre-training.
That said, the quartic dependence on $w(\pi^d)^{-1}$ in the worst case is a weakness of our approach.

Any pair of these first three results would be uninteresting on their own, but jointly, they show that with high probability, the imitator converges to the demonstrator with limited querying.

Our stronger results below apply when the environment and demonstrator policy do not depend on the query record. This means that whatever action is taken, the effect does not depend on whether the imitator chose it or the demonstrator did. We would like events to unfold similarly when we replace the demonstrator with the imitator, but this is impossible if the environment discriminates between them. Indeed, if the environment treats identical actions differently depending on whether they were selected by imitator or demonstrator, it's unclear what imitation accomplishes. We define fairness formally in Section \ref{sec:keyproofs}.

In a fair setting, we bound the KL divergence between $\ptrue$ and $\p^{\pi^d}_\mu$, the first meaning that actions are picked according to our imitation policy, and the second meaning that all actions are picked by the demonstrator. The objective of imitation is most easily characterized as outputting demonstrator-like actions, but the purpose of imitation learning is for events to unfold similarly. Small errors in the limit do not guarantee that property; this result is only possible with small errors for the imitator's whole lifetime.

\begin{restatable}[KL Bound]{theorem}{thmkl} \label{thm:kl}
Suppose that $\mu$ and $\pi^d$ are fair, and $\alpha < w(\pi^d)$. Letting the two probability measures below be restricted to $(\A \times \Ob)^t$ (that is, marginalizing over the query record, and considering only the first $t$ timesteps),
\begin{equation*}
    \KL_{t} \left(\ptrue(\cdot \mid E) \vb \vb \p^{\pi^d}_{\mu}(\cdot \mid E) \right) \leq \frac{\alpha^{-1} |\mathcal{A}|^{1/3} (24  w(\pi^d)^{-1} + 12)^{1/3}}{(1-\alpha/w(\pi^d))^2} t^{2/3} - \log(1-\alpha/w(\pi^d))
\end{equation*}
\end{restatable}

Notably, $\KL_t / t \to 0$ in the limit. The direction of the divergence resembles the variational objective (with the ground truth on the right). Thus, there may be some events that only the demonstrator would cause, but no events that only the imitator would. This consequence is made explicit in our final result.

We construct an upper bound for the probability of an event given the probability of the event if the demonstrator were acting the whole time. This bound is mainly of interest for ``bad'' events.

\begin{restatable}[Preserving Unlikeliness]{theorem}{thmbadevent} \label{thm:badevent}
Fix $t$. Let $B \subset (\A \times \Ob)^t$ be a (bad) event, and extending $B$ to the outcome space $(\{0, 1\} \times \A \times \Ob)^t = \mathcal{H}^t$, let $D = B \cap E$. Then, for fair $\mu$ and $\pi^d$,

\begin{equation*}
    \ptrue(D) \leq \frac{t^2 s_\alpha}{\left(\log \frac{t^2 s_\alpha}{27 \p^{\pi^d}_\mu(B)} - 3 \log \log \left(1 + \frac{t^{2/3} s_\alpha^{1/3}}{3 \p^{\pi^d}_\mu(B)^{1/3}}\right) \right)^3}
\end{equation*}
where $s_\alpha = |\A| \alpha^{-3}(24  w(\pi^d)^{-1} + 12)$.
\end{restatable}

That is, as $\p^{\pi^d}_\mu(B)^{-1} \to \infty$, $\ptrue(D)^{-1} \to \infty$ at least polylogarithmicly. If an event would have been extremely unlikely under the demonstrator's policy, a similar event is unlikely when running the imitator.

Whereas existing work on imitation learners attempts to be robust to a bounded loss function, our Preserving Unlikeliness Theorem is relevant even in the absence of a uniform bound on badness. In the real world, to quote Theon Greyjoy, ``It can always be worse''. But some bounds on badness are possible: we tolerate one-in-ten-chance events; they happen, and we get on with it. One-in-a-hundred-chance events can be meaningfully worse. But in a world largely governed by humans, we keep most truly devastating events below even a 1\% chance. It's hard to apply similar bounds to the badness of one-in-a-billion-chance events, and in general, as the probability gets smaller, a loss function should countenance steadily larger losses. When an event goes from a 1\% to a 2\% chance, we should be much less concerned than if it went from $10^{-9}$ to 1\%. In the extreme, if an event has probability $0$ under a demonstrator's policy, there might be an arbitrarily good reason for that. Whereas the bounded loss functions of all existing work ignore this effect, our Theorem \ref{thm:badevent} does not.

The main weaknesses of our results are what they require: a model class that includes the truth and a good choice of $\alpha$. Setting $\alpha$ well requires estimating $w(\pi^d)$, something we cannot offer general guidance on; it would depend entirely on the exact nature of the prior. And realistically, in many contexts, the realizability assumption is infeasible. There will always be mismatch between a computational model of a demonstrator and the true demonstrator. We hope this paper opens the door for other research into relaxing the realizability assumption. Plausibly, if the best approximation in $\Pi$ of $\pi^d$ produces certain bad events with low probability, then the imitator will too.

\section{Roadmap and Notation for the Proof of Theorem \ref{thm:queryprob}} \label{sec:proofnotation}
Much of the work of this paper is to prove Theorem \ref{thm:queryprob}. In this section, we state a theorem on which it depends, and we introduce the mathematical objects required to prove it.

The imitator queries when the top few demonstrator models disagree, so we bound the errors that those models can make over the agent's lifetime. We first must establish a finite bound on the errors of such models in ordinary Bayesian sequence prediction. We define that here.

Let $\mathcal{X}$ be an arbitrary finite alphabet. Let $\nu$ be a probability measure over $\mathcal{X}^\infty$ with the event space generated by the cylinder sets $\{\{x_{<t}x_{t:\infty} \ \mid \ x_{t:\infty} \in \mathcal{X}^\infty\}\mid x_{<t} \in \mathcal{X}^*\}$. Let $\M$ be a countable set of such probability measures, and let $w(\nu)$ be a prior weight over these measures such that $\sum_{\nu \in \M} w(\nu) = 1$. Let $x_{<t} \in \mathcal{X}^{t}$, let $\nu(x_{<t})$ denote the probability that the infinite sequence begins with $x_{<t}$, and let $\nu(x \mid x_{<t}) = \nu(x_{<t}x) / \nu(x_{<t})$. Let $\mu \in \M$ be a the ``true'' measure; that is, in formal results, we will let $x_{<\infty}$ be sampled from $\mu$.

Let $\nu_n^{x_{<t}}$ be the measure with the $n$\textsuperscript{th} largest posterior weight after observing $x_{<t}$; that is, order $\M$ to be non-increasing in $w(\nu)\nu(x_{<t})$, breaking ties arbitrarily, and take the $n$\textsuperscript{th}. (Ties between any pair should broken consistently for different $t$). Let the posterior $w(\nu \mid x_{<t}) :\propto w(\nu)\nu(x_{<t})$, normalized to sum to 1. Let $\M^{x_{<t}}_n$ be the set of the top $n$ measures, and let $w(\M^{x_{<t}}_n \mid x_{<t}) = \sum_{m \leq n} w(\nu \mid x_{<t})$.

Recall a model belongs to the imitator's top set if its posterior weight is at least $\alpha$ times the sum of the posterior weights of the models that are at least as good. Thus, we define
\begin{equation}
    \phi^{x_{<t}}_n := \frac{w(\nu_{n}^{x_{<t}} \mid x_{<t})}{w(\M_{n}^{x_{<t}} \mid x_{<t})}
\end{equation}
So if $\phi^{x_{<t}}_n \leq \alpha$, then $\nu^{x_{<t}}_n$ can be considered a ``top model'' in the same sense that is relevant to our imitation learner.

Our key result on which Theorem \ref{thm:queryprob} is based shows that taking the minimum over predictions in the top measures converges to the truth, and the ``missing probability'' converges to 0.

\begin{restatable}[Top Model Convergence]{theorem}{thmprederror} \label{thm:prederror}
\begin{align*}
    &\textrm{(i)} \hspace{1cm} \E_\mu \sum_{t = 0}^\infty \sum_{x \in \mathcal{X}} \left[\mu(x \mid x_{<t}) - \min_{n: \phi^{x_{<t}}_n > \alpha} \nu_n^{x_{<t}}(x \mid x_{<t})\right]^2 \leq \alpha^{-3}(24 w(\mu)^{-1} + 12)
    \\
    &\textrm{(ii)} \hspace{1cm} \E_\mu \sum_{t = 0}^\infty \left[1 - \sum_{x \in \mathcal{X}} \min_{n: \phi^{x_{<t}}_n > \alpha} \nu_n^{x_{<t}}(x \mid x_{<t})\right]^2 \leq |\mathcal{X}| \alpha^{-3}(24 w(\mu)^{-1} + 12)
\end{align*}
\end{restatable}

This is perfectly analogous to the way the imitator predicts actions: taking the minimum over the top models for which $\phi_n^{x_{<t}} > \alpha$. The difference is that in this sequence prediction setting, all observations are informative about the true measure, whereas the imitator rarely sees the demonstrator act.

To prove Theorem \ref{thm:prederror}, we show that a posterior-weighted mixture over $\M_n^{x_{<t}}$ converges to the truth, and if $\phi_n^{x_{<t}} > \alpha$, then each constituent must as well. This posterior-weighted mixture is called $\rhosn$. We define it here alongside other estimators that will be used in the proof of $\rhosn$'s convergence. First,

\begin{equation}
    \rhosn(x \mid x_{<t}) := \frac{\sum_{\nu \in \M_n^{x_{<t}}} w(\nu) \nu(x_{<t}x) }{ \sum_{\nu \in \M_n^{x_{<t}}} w(\nu) \nu(x_{<t})}
\end{equation}

$\rhosn$ resembles a maximum a posteriori estimate, but instead mixes over the top few. We call it a satis magnum a posteriori estimate (SMAP). We will show $\rhosn$ converges to $\rho_n$, which converges to $\rhonn$, which converges to $\mu$. $\rho_n$ and $\rhonn$ are alternative SMAP estimators.

$\rho_n$ is not a measure, as the numerator below sums over a different set than the denominator. It sums over the top measures \textit{after} observing $x$:
\begin{equation}
    \rho_n(x \mid x_{<t}) := \frac{\sum_{\nu \in \M_n^{x_{<t}x}} w(\nu) \nu(x_{<t}x) }{ \sum_{\nu \in \M_n^{x_{<t}}} w(\nu) \nu(x_{<t})}
\end{equation}
The definition appears more natural when considering a whole sequence:
\begin{equation} \label{eqn:rhoiseq}
    \rho_n(x_{<t}) = \sum_{\nu \in \M^{x_{<t}}_n} w(\nu) \nu(x_{<t})
\end{equation}
Since $\sum_{x \in \mathcal{X}} \rho_n(x \mid x_{<t})$ may not be 1, we construct the measure $\rhonn$ by normalizing:

\begin{equation} \label{def:rhoni}
    \rhonn(x \mid x_{<t}) := \frac{\rho_n(x \mid x_{<t})}{ \sum_{x' \in \mathcal{X}}\rho_n(x' \mid x_{<t})} = \frac{\rho_n(x_{<t}x)}{ \sum_{x' \in \mathcal{X}}\rho_n(x_{<t}x')}
\end{equation}

Our $\rho_n$, $\rhonn$, and $\rhosn$ are closely inspired by \citet{Hutter:05mdl2px}, who constructed (in our notation) $\rho_1$, $\rho^{\textrm{norm}}_1$, and $\rho^{\textrm{stat}}_1$. Finally, we define the full Bayes-mixture measure
\begin{equation}
    \xi(x_{<t}) := \sum_{\nu \in \M} w(\nu) \nu(x_{<t}) = \rho^{\textrm{stat}}_\infty(x_{<t}) = \rho_\infty(x_{<t}) = \rho^{\textrm{norm}}_\infty(x_{<t})
\end{equation}
We state those relationships without proof for the reader's interest; they are not used in our results.

\section{General Sequence Prediction Results} \label{sec:resultsseq}

This section organizes the proof of Theorem \ref{thm:prederror} into lemmas, some of which are proven here and some in Appendix \ref{sec:proofs}. We begin with elementary relations between $\xi$, $\rho_n$, $\rhonn$, and $\rhosn$.

\begin{align}
    \xi(x_{<t}) &\geq \rho_n(x_{<t}) \label{ineq:xirhoi}
    \\
    \rho_n(x_{<t}) &\geq w(\mu) \mu(x_{<t}) \label{ineq:rhoimu}
    \\
    \rho_n(x \mid x_{<t}) &\geq \rhonn(x \mid x_{<t}) \label{ineq:rhoirhoni}
    \\
    \rho_n(x \mid x_{<t}) &\geq \rhosn(x \mid x_{<t}) \label{ineq:rhoirhosi}
\end{align}

Inequalities \ref{ineq:xirhoi} and \ref{ineq:rhoimu} follow directly from Equation \ref{eqn:rhoiseq}. Inequality \ref{ineq:rhoirhoni} follows because \begin{multline} \label{eqn:antisem}
    \rho_n(x_{<t}) = \!\!\!\! \max_{\M' \subset \M : |\M'| = i} \sum_{\nu \in \M'} w(\nu) \nu(x_{<t}) = \!\!\!\! \max_{\M' \subset \M : |\M'| = i} \sum_{\nu \in \M'} w(\nu) \sum_{x \in \mathcal{X}} \nu(x_{<t}x) \\
    \leq \sum_{x \in \mathcal{X}} \max_{\M' \subset \M : |\M'| = i} \sum_{\nu \in \M'} w(\nu) \nu(x_{<t}x) = \sum_{x \in \mathcal{X}} \rho_n(x_{<t}x)
\end{multline}
so $\rho_n$ assigns too much probability mass. Inequality \ref{ineq:rhoirhosi} follows because
\begin{equation}
    \rhosn(x \mid x_{<t}) = \frac{\sum_{\nu \in \M_n^{x_{<t}}} w(\nu) \nu(x_{<t}x) }{ \sum_{\nu \in \M_n^{x_{<t}}} w(\nu) \nu(x_{<t})} \leq \frac{\sum_{\nu \in \M_n^{x_{<t}x}}w(\nu)\nu(x_{<t}x)}{\rho_n(x_{<t})} = \frac{\rho_n(x_{<t}x)}{\rho_n(x_{<t})}
\end{equation}
which holds because $\M_n^{x_{<t}x}$ is chosen to maximize the numerator.

Our first lemma bounds the normalizing factor for $\rho_n$, allowing us to show in our next lemma that it converges to both $\rhonn$ and $\rhosn$.

\begin{restatable}{lemma}{lemrhonorm} \label{lem:rhonorm}
\begin{equation*}
    0 \leq \E_\mu \sum_{t = 0}^\infty \frac{\sum_{x \in \mathcal{X}} \rho_n(x_{<t}x)}{\rho_n(x_{<t})} - 1 \leq w(\mu)^{-1}
\end{equation*}
\end{restatable}

\begin{proofidea}
$\rho_n$ is bounded above and below by measures, save a multiplicative constant (Inequalities \ref{ineq:xirhoi} and \ref{ineq:rhoimu}), so $\rho_n$ converges to being a measure, in that $\sum_{x \in \mathcal{X}} \rho_n(x \mid x_{<t}) \to 1$.
\end{proofidea}

\begin{proof}
All terms in the sum are non-negative, by Inequality \ref{eqn:antisem}. Recall $\epsilon$ denotes the empty string---the element of $\mathcal{X}^0$. Justifications of the upcoming lettered equations follow below the block.

\begin{align*}
    &\E_\mu \sum_{t = 0}^{N-1} \frac{\sum_{x \in \mathcal{X}} \rho_n(x_{<t}x)}{\rho_n(x_{<t})} - 1
    \\
    = &\sum_{t = 0}^{N-1} \sum_{x_{<t} \in \mathcal{X}^{t}} \mu(x_{<t}) \frac{\sum_{x \in \mathcal{X}} \rho_n(x_{<t}x) - \rho_n(x_{<t})}{\rho_n(x_{<t})} 
    \\
    \lequal^{(a)} &\sum_{t = 0}^{N-1} \sum_{x_{<t} \in \mathcal{X}^{t}} w(\mu)^{-1}\left[ \sum_{x \in \mathcal{X}} \rho_n(x_{<t}x) - \rho_n(x_{<t}) \right]
    \\
    \equal^{(b)} &w(\mu)^{-1} \left[ \sum_{x_{< N} \in \mathcal{X}^N} \rho_n(x_{< N}) - \rho_n(\epsilon) \right]
    \\
    \lequal^{(c)} &w(\mu)^{-1} \sum_{x_{< N} \in \mathcal{X}^N} \xi(x_{< N}) = w(\mu)^{-1}
    \tagaligneq
\end{align*}
where $(a)$ follows from Inequality \ref{ineq:rhoimu}, $(b)$ cancels terms that are added then subtracted, and $(c)$ follows from Inequality \ref{ineq:xirhoi}.
\end{proof}

Recall we are trying to show $\rhosn \to \rho_n \to \rhonn \to \mu$. The following lemma gives two of those links.

\begin{lemma} \label{lem:rhoi}
\begin{align*}
    &\textrm{(i)} \hspace{1cm} \E_\mu \sum_{t = 0}^\infty \sum_{x \in \mathcal{X}} \va \rho_n(x \mid x_{<t}) - \rhosn(x \mid x_{<t}) \va \leq w(\mu)^{-1}
    \\
    &\textrm{(ii)} \hspace{1cm} \E_\mu \sum_{t = 0}^\infty \sum_{x \in \mathcal{X}} \va \rho_n(x \mid x_{<t}) - \rhonn(x \mid x_{<t}) \va \leq w(\mu)^{-1}
\end{align*}
\end{lemma}

\begin{proof}
\begin{align*}
    &\E_\mu \sum_{t = 0}^\infty \sum_{x \in \mathcal{X}} \va \rho_n(x \mid x_{<t}) - \rhosn(x \mid x_{<t}) \va \equal^{(a)} \E_\mu \sum_{t = 0}^\infty \sum_{x \in \mathcal{X}} \rho_n(x \mid x_{<t}) - \rhosn(x \mid x_{<t}) =
    \\
    &\E_\mu \sum_{t = 0}^\infty \frac{\sum_{x \in \mathcal{X}} \rho_n(x_{<t}x)}{\rho_n(x_{<t})} - 1 \lequal^{(b)} w(\mu)^{-1}
    \tagaligneq
\end{align*}
where $(a)$ follows from Inequality \ref{ineq:rhoirhosi} and $(b)$ follows from Lemma \ref{lem:rhonorm}. The proof is identical for $\rhonn$, except now $(a)$ follows from Inequality \ref{ineq:rhoirhoni}.
\end{proof}

Given Lemma \ref{lem:rhoi}, the final link in showing $\rhosn$ converges to $\mu$ is to show that $\rhonn$ does.

\begin{restatable}{lemma}{lemmurhonorm}
\label{lem:murhonorm} Recalling $\nu(\cdot \mid x_{<t})$ is a measure over $\mathcal{X}$,
\begin{equation*}
    \E_\mu \sum_{t = 0}^\infty \KL\left(\mu(\cdot \mid x_{<t}) \va\va \rhonn(\cdot \mid x_{<t})\right) \leq w(\mu)^{-1} + \log w(\mu)^{-1} 
\end{equation*}
\end{restatable}
\begin{proofidea}
The KL divergence telescopes over timesteps. The $\log w(\mu)^{-1}$ term comes from a gap between $\mu$ and $\rho_n$, and the $w(\mu)^{-1}$ term comes from a gap between $\rho_n$ and $\rhonn$.
\end{proofidea}

We can now show that $\rhosn$ converges to $\mu$, an independently interesting and novel result in SMAP estimation.

\begin{restatable}[SMAP Convergence]{theorem}{thmrhosiconv} \label{thm:rhosiconv}
\begin{equation*}
    \E_\mu \sum_{t=0}^\infty \sum_{x \in \mathcal{X}} \left(\rhosn(x \mid x_{<t}) - \mu(x\mid x_{<t})\right)^2 \leq 6 w(\mu)^{-1} + 3
\end{equation*}
\end{restatable}

\begin{proofidea}
$\rhosn$ is close to $\rho_n$ in an $\ell_1$ sense, and likewise for $\rho_n$ and $\rhonn$, and $\rhonn$ is close to $\mu$ in an $\ell_2$ squared sense, since $\ell_2^2 \leq \KL$. Finally, for a vector $v \in [-1, 1]^n$, $||v||^2_2 \leq ||v||_1$, so $\ell_1$ proximity implies $\ell_2$ proximity as well.
\end{proofidea}

By applying Theorem \ref{thm:rhosiconv} to the very similar measures $\rhosn$ and $\rhosnm$, whose only difference is that the former contains $\nu^{x_{<t}}_n$ in its mixture, we arrive at our final result in the general sequence prediction setting.

\thmprederror*

\begin{proofidea}
$\rhosn$ is a weighted average of $\nu_m^{x_{<t}}$ for $m \leq n$, so convergence results for $\rhosn$ and $\rhosnm$ are leveraged for $\nu_n^{x_{<t}}$'s convergence. $\phi^{x_{<t}}_n > \alpha$ ensures the weights in the weighted average aren't too small, and that we only need to consider the top $\lfloor 1/\alpha \rfloor$ models.
\end{proofidea}

\section{Key Proofs} \label{sec:keyproofs}

We now prove our bound on the query probability, we define fairness, and we prove our bound on the probabilities of bad events.

\thmqueryprob*

\begin{proofidea}
The sort of model mismatch bounded by Theorem \ref{thm:prederror} (ii) is the basis for the definition of $\theta_q$. Theorem \ref{thm:prederror} (ii) bounds model mismatch on \textit{observed} data, and data is only observed with probability $\theta_q$, so with an extra factor of $\theta_q$ on the l.h.s., we go from an in mean square bound to a weaker in mean cube bound.
\end{proofidea}

\begin{proof}
Recall the agent considers a set of possible policies $\Pi$ that includes the true demonstrator policy $\pi^d$, and assigns a strictly positive prior $w(\pi)$ to each policy in $\Pi$. Recall $\p^\pi_\mu$ is a probability measure over $(\{0, 1\} \times \A \times \Ob)^\infty = \mathcal{H}^\infty$.
Now we construct a class of measures over $\mathcal{H}^\infty$: let $\M := \{\p^{\hat{\pi}_\alpha}_\mu : \pi \in \Pi\}$ (see the last paragraph of Section \ref{sec:imitation} for the definition of $\hat{\pi}_\alpha$), and let $w(\p^{\hat{\pi}_\alpha}_\mu) := w(\pi)$. Let $w(\p^{\hat{\pi}_\alpha}_\mu \mid h_{<t}) :\propto w(\p^{\hat{\pi}_\alpha}_\mu) \p^{\hat{\pi}_\alpha}_\mu(h_{<t})$. It follows straightforwardly from the definitions of the posterior that $w(\p^{\hat{\pi}_\alpha}_\mu \mid h_{<t}) = w(\pi \mid h_{<t})$, $w(\p^{\hat{\pi}_\alpha}_\mu \mid h_{<t}q_t) = w(\pi \mid h_{<t}q_t)$, and $w(\p^{\hat{\pi}_\alpha}_\mu \mid h_{<t}q_t a_t) = w(\pi \mid h_{<t}q_t a_t)$, since all measures in $\M$ assign the probabilities identically to actions after $q_t=0$, and to observations.

Instead of saying $\M$ contains measures over $\mathcal{X}^\infty$, we generalize slightly, and say that $\M$ contains measures over $\bigtimes_{k=0}^\infty \mathcal{X}_k$. For $k \equiv 0 \mod 3$, $\mathcal{X}_k = \{0, 1\}$, for $k \equiv 1 \mod 3$, $\mathcal{X}_k = \A$, and for $k \equiv 2 \mod 3$, $\mathcal{X}_k = \Ob$. With $\nu^{x_{<k}}_n$ and $\phi^{x_{<k}}_n$ as defined before, we can apply Theorem \ref{thm:prederror} (i) to the class $\M$, after a trivial extension from fixed $\mathcal{X}$ to variable $\mathcal{X}_k$. Checking the definitions is enough to verify that $\{\nu^{x_{<k}}_n : \phi^{x_{<k}}_n > \alpha\}$ is exactly the set $\{\p^{\hat{\pi}_\alpha}_\mu : \pi \in \pa\}$, where $h_j = (q_j, a_j, o_j) = (x_{3j}, x_{3j+1}, x_{3j+2})$, and $t = \lfloor (k+1)/3 \rfloor$. In short, for this $\M$, sequence prediction errors can only come from errors predicting actions after querying, since that's when models differ, so we can use Theorem \ref{thm:prederror} to bound the latter. Recalling that $\ptrue$ is the true probability measure,

\begin{align*}
    \alpha^{-3}(24  w(\pi^d)^{-1} + 12) &= \alpha^{-3}(24  w(\ptrue)^{-1} + 12)
    \\
    &\gequal^{(a)} \Etrue \sum_{k=0}^\infty \sum_{x \in \mathcal{X}_k} \left[\ptrue(x \mid x_{<k}) - \min_{i : \phi^{x_{<k}}_n > \alpha} \nu^{x_{<k}}_n(x \mid x_{<k}) \right]^2
    \\
    &\equal^{(b)} \Etrue \sum_{t=0}^\infty \sum_{q \in \{0, 1\}} \left[\ptrue(q \mid h_{<t}) - \min_{\pi \in \pa} \p^{\hat{\pi}_\alpha}_\mu(q \mid h_{<t}) \right]^2 +
    \\
    &\hspace{1.6cm} \sum_{a \in \A} \left[\ptrue(a \mid h_{<t}q_t) - \min_{\pi \in \pa} \p^{\hat{\pi}_\alpha}_\mu(a \mid h_{<t}q_t) \right]^2 +
    \\
    &\hspace{1.6cm} \sum_{o \in \Ob} \left[\ptrue(o \mid h_{<t}q_ta_t) - \min_{\pi \in \pa} \p^{\hat{\pi}_\alpha}_\mu(o \mid h_{<t}q_ta_t) \right]^2
    \\
    &\equal^{(c)} \Etrue \sum_{t=0}^\infty \sum_{a \in \A} \left[\ptrue(a \mid h_{<t}q_t) - \min_{\pi \in \pa} \p^{\hat{\pi}_\alpha}_\mu(a \mid h_{<t}q_t) \right]^2
    \\
    &= \Etrue \sum_{t=0}^\infty \sum_{q \in \{0, 1\}} \ptrue(q \mid h_{<t}) \sum_{a \in \A} \left[\ptrue(a \mid h_{<t}q) - \min_{\pi \in \pa} \p^{\hat{\pi}_\alpha}_\mu(a \mid h_{<t}q) \right]^2
    \\
    &\equal^{(d)} \Etrue \sum_{t=0}^\infty \ptrue(1 \mid h_{<t}) \sum_{a \in \A} \left[\ptrue(a \mid h_{<t}1) - \min_{\pi \in \pa} \p^{\hat{\pi}_\alpha}_\mu(a \mid h_{<t}1) \right]^2
    \\
    &\gequal^{(e)}
    \Etrue \sum_{t=0}^\infty \theta_q(h_{<t}) |\A| \left[|\A|^{-1 }\sum_{a \in \A} \ptrue(a \mid h_{<t}1) - \min_{\pi \in \pa} \p^{\hat{\pi}_\alpha}_\mu(a \mid h_{<t}1) \right]^2
    \\
    &= |\A|^{-1} \Etrue \sum_{t=0}^\infty \theta_q(h_{<t}) \left[1 - \sum_{a \in \A} \min_{\pi \in \pa} \frac{\hat{\pi}_\alpha(1, a \mid h_{<t})}{\hat{\pi}_\alpha(1 \mid h_{<t})} \right]^2
    \\
    &\equal^{(f)} |\A|^{-1} \Etrue \sum_{t=0}^\infty \theta_q(h_{<t}) \left[1 - \sum_{a \in \A} \min_{\pi \in \pa} \frac{\theta_q(h_{<t}) \pi(1, a \mid h_{<t})}{\theta_q(h_{<t})} \right]^2
    \\
    &\equal^{(g)} |\A|^{-1} \Etrue \sum_{t=0}^\infty \theta_q(h_{<t}) \left[\theta_q(h_{<t}) \right]^2
    \label{ineq:thetaq} \tagaligneq
\end{align*}
where $(a)$ follows from Theorem \ref{thm:prederror}, $(b)$ groups triples $(x_{3t}, x_{3t+1}, x_{3t+2})$ into $h_t$, $(c)$ follows because all $\p^{\hat{\pi}_\alpha}_\mu \in \M$ give identical conditional probabilities as $\ptrue$ on queries and observations, $(d)$ follows because all $\p^{\hat{\pi}_\alpha}_\mu \in \M$ give identical conditional probabilities as $\ptrue$ for actions that follow $q_t = 0$, $(e)$ follows from Jensen's Inequality, $(f)$ follows from the definition of $\hat{\pi}_\alpha$, and $(g)$ follows from the definition of $\theta_q(h_{<t})$. Rearranging Inequality \ref{ineq:thetaq} gives the theorem.
\end{proof}

Recall that Theorem \ref{thm:imtodem} bounds the error between $\pi^i_\alpha$ and $\pi^d$, conditioned on the event $E$.

\begin{proofidea}[Proof idea of Theorem \ref{thm:imtodem}]
Conditioned on $\pi^d \in \pa$, it follows that $\theta_q \geq$ the $\ell_1$ norm between $\pi^d$ and $\pi^i_\alpha$. Then we apply Theorem \ref{thm:queryprob}.
\end{proofidea}

Our remaining theorems apply when the environment and demonstrator policy are fair. Roughly, they are fair if they do not have access to the imitator's internals.

\begin{definition}[Fair]
An environment $\mu : \mathcal{H}^* \times \{0, 1\} \times \A \rightsquigarrow \Ob$ is fair if it does not depend on the query record; that is, $\mu(\cdot \mid h_{<t} q_t a_t)$ is not a function of $q_k$ for $k \leq t$. A demonstrator policy $\pi^d : \mathcal{H}^* \rightsquigarrow \{0, 1\} \times \A$ is likewise fair if $\pi^d(\cdot \mid h_{<t})$ is not a function of $q_k$ for $k < t$.
\end{definition}

Theorems \ref{thm:kl} and \ref{thm:badevent} rest on the following crux: if $\pi^d \in \pa$, then $\pi^i_\alpha(0, a \mid h_{<t}) \leq \pi^d(a \mid h_{<t})$. Since $\pi^i_\alpha(1, a \mid h_{<t}) = \theta_q(h_{<t}) \pi^d(a \mid h_{<t})$, we have $\pi^i_\alpha(a \mid h_{<t}) \leq (1 + \theta_q(h_{<t})) \pi^d(a \mid h_{<t})$. Thus, we have a multiplicative bound relating $\pi^i_\alpha$ and $\pi^d$, and it decreases to 1.

\thmbadevent*

\begin{proofidea}
$\ptrue(B \cap E)/\p^{\pi^d}_\mu(B)$ increases by a factor of at most $1 + \theta_q$ per timestep. While the expectation of $\theta_q^3$ is summable, the expectation of $\sum_t \theta_q$ grows as $O(t^2)$, hence that dependence in the bound. The final difficulty is that our bound on the query probability only applies in expectation, but a pathological and unlikely event $B$ could describe a case where querying is much more prolonged than expected. Thus, we do not prove a nice bound on the ratio $\ptrue(B \cap E)/\p^{\pi^d}_\mu(B)$. Instead, since smaller $\p^{\pi^d}_\mu(B)$ allows more pathology, our bound on $\ptrue(B \cap E)$ is only polylogarithmic in $\p^{\pi^d}_\mu(B)$.
\end{proofidea}

\begin{proof}
If $\pi^d \in \pa$, then $\pi^i_\alpha(0, a \mid h_{<t}) \leq \pi^d(a \mid h_{<t})$, and of course $\pi^i_\alpha(1, a \mid h_{<t}) = \theta_q(h_{<t}) \pi^d(a \mid h_{<t})$, so
\begin{equation} \label{ineq:piipid}
\pi^i_\alpha(a \mid h_{<t}) \leq (1 + \theta_q(h_{<t}))\pi^d(a \mid h_{<t})
\end{equation}
Thus, for fair $\mu$ and $\pi^d$, for $h_{<t} \in E$,
\begin{equation} \label{ineq:ratioboundintheta}
    \frac{\ptrue(h^{\setminus}_{<t})}{\p^{\pi^d}_\mu(h^{\setminus}_{<t})} \leq \prod_{k=0}^{t-1} [1 + \theta_q(h_{<k})]
\end{equation}
It follows from Theorem \ref{thm:queryprob} that
\begin{equation}
    \Etrue \left[\sum_{k=0}^{t-1} \theta_q(h_{<k})^3 \vc D \right] \leq \frac{s_\alpha}{\ptrue(D)}
\end{equation}
By the same derivation as in Inequality \ref{ineq:jensencube}, we can thus bound the sum
\begin{equation} \label{ineq:sixty}
    \Etrue \left[\sum_{k=0}^{t-1} \theta_q(h_{<k}) \vc D \right] \leq t^{2/3} \left( \frac{s_\alpha}{\ptrue(D)} \right)^{1/3}
\end{equation}
Now, applying Inequality \ref{ineq:piipid} repeatedly, 
\begin{align*}
    &\Etrue \left[\prod_{k=0}^{t-1}(1 + \theta_q(h_{<k}))^{-1} \vc D \right] 
    \\
    &= \frac{\sum_{h_{<t} \in D} \ptrue(h_{<t}) \prod_{k=0}^{t-1}(1 + \theta_q(h_{<k}))^{-1}}{\sum_{h_{<t} \in D} \ptrue(h_{<t})}
    \\
    &\leq \frac{\sum_{h_{<t-1} \in E} \sum_{h_{t-1} \in \mathcal{H} : h^{\setminus}_{<t} \in B} \ptrue(h_{<t}) \prod_{k=0}^{t-1}(1 + \theta_q(h_{<k}))^{-1}}{\sum_{h_{<t} \in D} \ptrue(h_{<t})}
    \\
    &= \frac{\sum_{h_{<t-1} \in E} \left[\ptrue(h_{<t-1}) \prod_{k=0}^{t-2}(1 + \theta_q(h_{<k}))^{-1} \right] \sum_{h^{\setminus}_{t-1} \in \A \times \Ob : h^{\setminus}_{<t} \in B} \ptrue(h^{\setminus}_{t-1} \mid h_{<t-1}) (1 + \theta_q(h_{<t-1}))^{-1}}{\sum_{h_{<t} \in D} \ptrue(h_{<t})}
    \\
    &\lequal^{(a)} \frac{\sum_{h_{<t-1} \in E} \left[\ptrue(h_{<t-1}) \prod_{k=0}^{t-2}(1 + \theta_q(h_{<k}))^{-1} \right] \sum_{h^{\setminus}_{t-1} \in \A \times \Ob : h^{\setminus}_{<t} \in B} \p^{\pi^d}_\mu(h^{\setminus}_{t-1} \mid h_{<t-1})}{\sum_{h_{<t} \in D} \ptrue(h_{<t})}
    \\
    &\lequal^{(b)} \frac{\sum_{h_{<t-2} \in E} \left[\ptrue(h_{<t-2}) \prod_{k=0}^{t-3}(1 + \theta_q(h_{<k}))^{-1} \right] \sum_{h^{\setminus}_{t-2} h^{\setminus}_{t-1} \in (\A \times \Ob)^2 : h^{\setminus}_{<t} \in B} \p^{\pi^d}_\mu(h^{\setminus}_{t-2} h^{\setminus}_{t-1} \mid h_{<t-2})}{\sum_{h_{<t} \in D} \ptrue(h_{<t})}
    \\
    &\lequal^{(c)} \frac{\sum_{h^{\setminus}_{<t} \in B} \p^{\pi^d}_\mu(h^{\setminus}_{<t})}{\sum_{h_{<t} \in D} \ptrue(h_{<t})} = \frac{\p^{\pi^d}_\mu(B)}{\ptrue(D)}
    \tagaligneq
\end{align*}
where $(a)$ follows from Inequality \ref{ineq:piipid} since $h_{<t-1} \in E$ (note the change from $\pi^i_\alpha$ to $\pi^d$), $(b)$ iterates the previous three lines, and $(c)$ iterates the logic down to 0.

Now we bound the expectation
\begin{align*} \label{ineq:sixtyone}
    \Etrue\left[\prod_{k=0}^{t-1}(1 + \theta_q(h_{<k}))^{-1} \vc D \right]
    &\gequal^{(a)} 
    \prod_{k=0}^{t-1}(1 + \Etrue\left[ \theta_q(h_{<k}) \mid D \right])^{-1}
    \\
    &= \mathrm{exp}\left(-\sum_{k=0}^{t-1} \log \left(1 + \Etrue\left[ \theta_q(h_{<k}) \mid D \right]\right) \right)
    \\
    &\geq \mathrm{exp}\left(-\sum_{k=0}^{t-1} \Etrue\left[ \theta_q(h_{<k}) \mid D \right] \right)
    \\
    &\gequal^{(b)}
    \tagaligneq e^{-t^{2/3} s_\alpha^{1/3} \ptrue(D)^{-1/3}}
\end{align*}
where $(a)$ follows from Jensen's Inequality (one can easily show the Hessian of $\prod_i 1/(1+x_i)$ is positive semidefinite for $x \succ 0$), and $(b)$ follows from Inequality \ref{ineq:sixty}. Solving for $\ptrue(D)$ in terms of $\p^{\pi^d}_\mu(B)$, we get
\begin{equation}
    \ptrue(D) \leq \frac{t^2 s_\alpha}{27 W(\frac{t^{2/3} s_\alpha^{1/3}}{3 \p^{\pi^d}_\mu(B)^{1/3}})^3}
\end{equation}
where $W$ is the Lambert-$W$ function, defined by the property $W(z) e^{W(z)} = z$.
A property of the Lambert-$W$ function---that $W(z) \geq \log z - \log \log (1 + z)$---yields the theorem:
\begin{equation*}
    \ptrue(D) \leq \frac{t^2 s_\alpha}{\left(\log \frac{t^2 s_\alpha}{27 \p^{\pi^d}_\mu(B)} - 3 \log \log \left(1 + \frac{t^{2/3} s_\alpha^{1/3}}{3 \p^{\pi^d}_\mu(B)^{1/3}}\right) \right)^3}
\end{equation*}
One can easily verify this inequality by supposing the opposite and showing that it violates Inequality \ref{ineq:sixtyone}, but we omit this.
\end{proof}

\section{Conclusion}

We present the first formal results for an imitation learner in a setting where the environment does not reset. We present the first formal results for an imitation learner that do not depend on a bounded loss assumption. We present the first finite error bounds for an agent acting in general environments; existing results only regard limiting behavior (although existing work considers reinforcement learning, a harder problem than imitation learning). If we would like to have an artificial agent imitate, with particular concern for keeping unlikely events unlikely, this is the first theory of how to do it.

\begin{acks}
This work was supported by the Leverhulme Centre for the Future of Intelligence,  Australian Research
Council Discovery Projects DP150104590, the Oxford-Man Institute, and the Berkeley Existential Risk Initiative. Thank you to Jan Leike for encouraging us to write a paper on imitation learning.
\end{acks}

\bibliography{cohen}

\appendix

\section{Notation and Definitions} \label{sec:notationtable}
\noindent\begin{tabular}{|p{2cm}|p{\linewidth-29mm}|}
\hline
\textbf{Notation} & \textbf{Meaning} \\ \hline
\multicolumn{2}{|l|}{\textit{Preliminary Notation}} \\ \hline
$\A$, $\Ob$ & the finite action/observation spaces \\ \hline
$a_t$, $o_t$ & $\in \A, \Ob$; the action and observation at timestep $t$\\ \hline
$q_t$ & $\in \{0, 1\}$; indicates whether the demonstrator is queried at time $t$ \\ \hline
$\mathcal{H}$ & $\{0, 1\} \times \A \times \Ob$ \\ \hline
$h_{t}$ & $(q_t, a_t, o_t)$; the interaction history in the $t$\textsuperscript{th} timestep \\ \hline
$h_{<t}$ & $(h_1, ..., h_{t-1})$ \\ \hline
$h_t^{\setminus}$ & $(a_t, o_t)$ \\ \hline
$\epsilon$ & the empty history \\ \hline
$\pi$ & policy stochastically mapping $\mathcal{H}^* \rightsquigarrow \{0, 1\} \times \A$ \\ \hline
$\mu$ & environment stochastically mapping $\mathcal{H}^* \times \{0, 1\} \times \A \rightsquigarrow \mathcal{O}$ \\ \hline
$\p^\pi_\nu$ & a probability measure over histories with actions sampled from $\pi$ and observations sampled from $\nu$ \\ \hline
$\E^\pi_\nu$ & the expectation when the interaction history is sampled from $\p^\pi_\nu$ \\ \hline
$w(\pi)$ & (positive) prior weight that the policy $\pi$ is the demonstrator's \\ \hline
$w(\pi \mid h_{<t})$ & posterior weight on the policy $\pi$; $\propto w(\pi) \prod_{k < t : q_k = 1} \pi(q_k a_k \mid h_{<k})$ \\ \hline
\multicolumn{2}{|l|}{\textit{Imitation Learner Definition}} \\ \hline
$\alpha$ & $\in (0, 1]$; lower values mean the imitator better resembles the demonstrator, but queries longer \\ \hline
$\pa$ & set of top models; $\{\pi^{h_{<t}}_n \in \Pi : w(\pi^{h_{<t}}_n \mid h_{<t}) \geq \alpha \sum_{m \leq n} w(\pi^{h_{<t}}_m \mid h_{<t})\}$ \\ \hline
$\pi^d$ & the demonstrator's policy \\ \hline
$\pi^i_\alpha$ & the imitator's policy; $\pi^i_\alpha(0, a \mid h_{<t}) = \min_{\pi' \in \pa} \pi'(1, a \mid h_{<t})$, and $\pi^i_\alpha(1, a \mid h_{<t}) = \theta_q(h_{<t}) \pi^d(1, a \mid h_{<t})$ \\ \hline
$\theta_q(h_{<t})$ & the query probability; $1 - \sum_{a \in \A} \pi^i_\alpha(0, a \mid h_{<t})$ \\ \hline
$\hat{\pi}_\alpha$ & the imitator policy defined with respect to an arbitrary demonstrator $\pi$, not the real demonstrator $\pi^d$ \\ \hline
\multicolumn{2}{|l|}{\textit{General Sequence Prediction}} \\ \hline
$\mathcal{X}$ & finite alphabet \\ \hline
$x_{<t}$ & an element of $\mathcal{X}^t$ \\ \hline
$\M$ & countable set of measures over $\mathcal{X}^\infty$ \\ \hline
$w(\nu)$ & prior weight on $\nu \in \M$ \\ \hline
$w(\nu \mid x_{<t})$ & posterior weight on $\nu \in \M$ \\ \hline
$\xi$ & $\xi(x_{<t}) = \sum_{\nu \in \M} w(\nu) \nu(x_{<t})$ \\ \hline
$\rho_n$ & $\rho_n(x_{<t}) = \max_{\M' \subset \M : |\M'| = i} \sum_{\nu \in \M'} w(\nu) \nu(x_{<t})$ \\ \hline
$\rhonn$ & like $\rho_n$, but normalized to be a measure \newline $\rhonn(x \mid x_{<t}) = \rho_n(x \mid x_{<t}) /  \sum_{x' \in \mathcal{X}}\rho_n(x' \mid x_{<t})$ \\ \hline
$\M_n^{x_{<t}}$ & $\argmax_{\M' \subset \M : |\M'| = i} \sum_{\nu \in \M} w(\nu) \nu(x_{<t})$ \\ \hline
$\rhosn$ & a mixture over the top $i$ models, sorted by posterior weight \newline $\rhosn(x \mid x_{<t}) = \sum_{\nu \in \M_n^{x_{<t}}} w(\nu) \nu(x_{<t}x) / \sum_{\nu \in \M_n^{x_{<t}}} w(\nu) \nu(x_{<t})$ \\ \hline
$\phi^{x_{<t}}_n$ & $w(\nu_n^{x_{<t}} \mid x_{<t}) / w(\M_n^{x_{<t}} \mid x_{<t})$ \\ \hline
\end{tabular}

\section{Omitted Proofs} \label{sec:proofs}

\lemmurhonorm*
\begin{proof}
The KL divergence is non-negative, so we bound an arbitrary finite sum.
\begin{align*}
    &\E_\mu \sum_{t = 0}^{N-1} \KL\left(\mu(\cdot \mid x_{<t}) \va\va \rhonn(\cdot \mid x_{<t})\right)
    \\
    = &\sum_{t = 0}^{N-1} \E_\mu \sum_{x_t \in \mathcal{X}} \mu(x_t \mid x_{<t}) \log \frac{\mu(x_t \mid x_{<t})}{\rhonn(x_t \mid x_{<t})}
    \\
    \equal^{(a)} &\sum_{t = 0}^{N-1} \sum_{x_{<t} \in \mathcal{X}^{t}} \mu(x_{<t}) \sum_{x_t \in \mathcal{X}} \mu(x_t \mid x_{<t}) \left[ \log \frac{\mu(x_t \mid x_{<t})}{\rho_n(x_t \mid x_{<t})} + \log \frac{\sum_{x' \in \mathcal{X}}\rho_n(x_{<t}x')}{\rho_n(x_{<t})} \right]
    \\
    \lequal^{(b)} &w(\mu)^{-1} + \sum_{t = 0}^{N-1} \sum_{x_{<t} \in \mathcal{X}^{t}} \mu(x_{<t}) \sum_{x_t \in \mathcal{X}} \mu(x_t \mid x_{<t}) \log \frac{\mu(x_t \mid x_{<t})}{\rho_n(x_t \mid x_{<t})}
    \\
    = &w(\mu)^{-1} + \sum_{t = 0}^{N-1} \sum_{x_{<t} \in \mathcal{X}^{t}} \mu(x_{<t}) \sum_{x_t \in \mathcal{X}} \mu(x_t \mid x_{<t}) \left[\log \frac{\mu(x_{<t} x_t)}{\rho_n(x_{<t}x_t)} - \log \frac{\mu(x_{<t})}{\rho_n(x_{<t})}\right]
    \\
    = &w(\mu)^{-1} + \sum_{t = 0}^{N-1} \sum_{x_{<t} \in \mathcal{X}^{t}} \mu(x_{<t}) \left[\sum_{x_t \in \mathcal{X}} \mu(x_t \mid x_{<t}) \log \frac{\mu(x_{<t} x_t)}{\rho_n(x_{<t}x_t)} - \log \frac{\mu(x_{<t})}{\rho_n(x_{<t})}\right]
    \\
    = &w(\mu)^{-1} + \sum_{t = 0}^{N-1} \left[ \sum_{x_{\leq t} \in \mathcal{X}^{t+1}} \mu(x_{\leq t}) \log \frac{\mu(x_{\leq t})}{\rho_n(x_{\leq t})} - \sum_{x_{<t} \in \mathcal{X}^{t}} \mu(x_{<t}) \log \frac{\mu(x_{<t})}{\rho_n(x_{<t})} \right]
    \\
    \equal^{(c)} &w(\mu)^{-1} + \sum_{x_{< N} \in \mathcal{X}^N} \mu(x_{< N}) \log \frac{\mu(x_{< N})}{\rho_n(x_{< N})} - \mu(\epsilon) \log \frac{\mu(\epsilon)}{\rho_n(\epsilon)}
    \\
    \lequal^{(d)} &w(\mu)^{-1} + \sum_{x_{< N} \in \mathcal{X}^N} \mu(x_{< N}) \log w(\mu)^{-1} = w(\mu)^{-1} + \log w(\mu)^{-1}
    \tagaligneq
\end{align*}
where $(a)$ follows from the definition of $\rhonn$ in Equation \ref{def:rhoni}, $(b)$ follows from Lemma \ref{lem:rhonorm} and the fact that $\log x \leq x - 1$, $(c)$ cancels like terms, and $(d)$ follows from Inequality \ref{ineq:rhoimu}.
\end{proof}

\thmrhosiconv*

\begin{proof}
We abbreviate $w(\mu)^{-1}$ as $c$. Let $[N] := (0, ..., N-1)$. We define an $N|\mathcal{X}|$-dimensional random vector depending on the infinite sequence $x_{<\infty}$:
\begin{equation}
    \overrightarrow{\nu_1\!\!:\!\!\nu_2}^N := \left( \nu_1(x \mid x_{<t}) - \nu_2(x\mid x_{<t}) \right)_{t \in [N], x \in \mathcal{X}}
\end{equation}

In this notation, we aim to show $\E_\mu ||\overrightarrow{\rhosn\!\!:\!\!\mu}^N||^2_2 \leq 6c + 3$. Lemma \ref{lem:rhoi} (i) and (ii) become
\begin{align}
    &\E_\mu ||\overrightarrow{\rho_n\!\!:\!\!\rhosn}^N||_1 \leq c
    \\
    &\E_\mu ||\overrightarrow{\rho_n\!\!:\!\!\rhonn}^N||_1 \leq c
\end{align}
Therefore,
\begin{equation}
    \E_\mu ||\overrightarrow{\rhosn\!\!:\!\!\rhonn}^N||_1 \leq 2c
\end{equation}

Since each element in this vector is in $[-1, 1]$, squaring them makes the magnitude no larger, so
\begin{equation}
    \E_\mu ||\overrightarrow{\rhosn\!\!:\!\!\rhonn}^N||^2_2 \leq 2c
\end{equation}

The KL divergence is larger than the sum of the squares of the probability differences (proven, for example, in \cite[\S 3.9.2]{Hutter:04uaibook}), so Lemma \ref{lem:murhonorm} implies
\begin{equation}
    \E_\mu ||\overrightarrow{\rhonn\!\!:\!\!\mu}^N||^2_2 \leq c + \log c
\end{equation}
By the triangle inequality,
\begin{equation}
    ||\overrightarrow{\rhosn\!\!:\!\!\mu}^N||_2 \leq ||\overrightarrow{\rhosn\!\!:\!\!\rhonn}^N||_2 + ||\overrightarrow{\rhonn\!\!:\!\!\mu}^N||_2
\end{equation}
so
\begin{equation}
    ||\overrightarrow{\rhosn\!\!:\!\!\mu}^N||^2_2 \leq
    ||\overrightarrow{\rhosn\!\!:\!\!\rhonn}^N||^2_2 + ||\overrightarrow{\rhonn\!\!:\!\!\mu}^N||^2_2 +
    2 ||\overrightarrow{\rhosn\!\!:\!\!\rhonn}^N||_2 ||\overrightarrow{\rhonn\!\!:\!\!\mu}^N||_2
\end{equation}
and because $\E[XY] \leq \sqrt{\E [X^2]\E [Y^2]}$ (the Cauchy--Schwarz Inequality),
\begin{equation}
    \E_\mu ||\overrightarrow{\rhosn\!\!:\!\!\mu}^N||^2_2 \leq 2c + (c + \log c) + 2\sqrt{2c(c+\log c)} < 6c + 3
\end{equation}
\end{proof}

We name the measure with the $i$\textsuperscript{th} largest posterior weight
\begin{equation}
    \nu^{x_{<t}}_n :\in \M^{x_{<t}}_n \setminus \M^{x_{<t}}_{i-1}
\end{equation}
with the posterior weight formally defined $w(\nu \mid x_{<t}) := \frac{w(\nu)\nu(x_{<t})}{\xi(x_{<t})}$,
and $w(\M' \mid x_{<t}) := \sum_{\nu \in \M'} w(\nu \mid x_{<t})$. Now, we let 
\begin{equation}
    \phi^{x_{<t}}_n := \frac{w(\nu_n^{x_{<t}} \mid x_{<t})}{w(\M_n^{x_{<t}} \mid x_{<t})}
\end{equation}

\thmprederror*
\begin{proof}
$\rhosn(x\mid x_{<t})$ is a weighted average of $\nu_j^{x_{<t}}(x\mid x_{<t})$ for $j \leq i$:
\begin{align*}
    \rhosn(x\mid x_{<t}) &= \frac{\sum_{\nu \in \M_n^{x_{<t}}} w(\nu) \nu(x_{<t}x) }{ \sum_{\nu \in \M_n^{x_{<t}}} w(\nu) \nu(x_{<t})}
    \\
    &= \frac{\sum_{\nu \in \M_n^{x_{<t}}} w(\nu) \nu(x_{<t})\nu(x \mid x_{<t}) }{ \sum_{\nu \in \M_n^{x_{<t}}} w(\nu) \nu(x_{<t})}
    \\
    &= \frac{\sum_{\nu \in \M_n^{x_{<t}}} w(\nu \mid x_{<t}) \xi(x_{<t})\nu(x \mid x_{<t}) }{ \sum_{\nu \in \M_n^{x_{<t}}}  w(\nu \mid x_{<t}) \xi(x_{<t})}
    \\
    &= \sum_{\nu \in \M_n^{x_{<t}}} \frac{w(\nu \mid x_{<t}) }{w(\M_n^{x_{<t}} \mid x_{<t})} \nu(x \mid x_{<t})
    \\
    &= \sum_{j = 1}^i \frac{w(\nu_j^{x_{<t}} \mid x_{<t}) }{ w(\M_n^{x_{<t}} \mid x_{<t})} \nu_j^{x_{<t}}(x\mid x_{<t})
    \tagaligneq
\end{align*}
Trivially,
\begin{equation}
    \nu^{x_{<t}}_1(x\mid x_{<t}) = \rhoso(x \mid x_{<t})
\end{equation}
but for $i > 1$, we would like to express $\nu^{x_{<t}}_n$ in terms of $\rhosn$ and $\rhosnm$:
\begin{equation}
    \rhosn(x \mid x_{<t}) = \frac{w(\M_{i-1}^{x_{<t}} \mid x_{<t})}{w(\M_n^{x_{<t}} \mid x_{<t})} \rhosnm(x \mid x_{<t}) + \frac{w(\nu_n^{x_{<t}} \mid x_{<t})}{w(\M_n^{x_{<t}} \mid x_{<t})} \nu^{x_{<t}}_n(x \mid x_{<t})
\end{equation}
Thus,
\begin{equation} \label{eqn:nui}
    \nu^{x_{<t}}_n(x\mid x_{<t}) = \frac{w(\M_n^{x_{<t}} \mid x_{<t})}{w(\nu_n^{x_{<t}} \mid x_{<t})}\rhosn(x \mid x_{<t}) - \frac{w(\M_{i-1}^{x_{<t}} \mid x_{<t})}{w(\nu_n^{x_{<t}} \mid x_{<t})}\rhosnm(x \mid x_{<t})
\end{equation}
Since $\frac{w(\M_n^{x_{<t}} \mid x_{<t})}{w(\nu_n^{x_{<t}} \mid x_{<t})} -  \frac{w(\M_{i-1}^{x_{<t}} \mid x_{<t})}{w(\nu_n^{x_{<t}} \mid x_{<t})} = 1$, for $i > 1$,
\begin{multline}
    \nu^{x_{<t}}_n(x\mid x_{<t}) - \mu(x\mid x_{<t}) = \frac{w(\M_n^{x_{<t}} \mid x_{<t})}{w(\nu_n^{x_{<t}} \mid x_{<t})}\left[\rhosn(x \mid x_{<t}) - \mu(x\mid x_{<t})\right]-
    \\
    \frac{w(\M_{i-1}^{x_{<t}} \mid x_{<t})}{w(\nu_n^{x_{<t}} \mid x_{<t})}\left[ \rhosnm(x \mid x_{<t}) - \mu(x\mid x_{<t}) \right]
\end{multline}
Recall
\begin{equation*}
    \phi^{x_{<t}}_n := \frac{w(\nu_n^{x_{<t}} \mid x_{<t})}{w(\M_n^{x_{<t}} \mid x_{<t})}
\end{equation*}
Since $w(\M_{i-1}^{x_{<t}} \mid x_{<t}) \leq w(\M_n^{x_{<t}} \mid x_{<t})$, we have
\begin{multline}
    (\phi^{x_{<t}}_n)^2 \left[\nu^{x_{<t}}_n(x\mid x_{<t}) - \mu(x\mid x_{<t})\right]^2 \leq 2 \left[\rhosn(x \mid x_{<t}) - \mu(x\mid x_{<t})\right]^2 +
    \\
    2 \left[ \rhosnm(x \mid x_{<t}) - \mu(x\mid x_{<t}) \right]^2
\end{multline}

Now we consider all measures $\nu^{x_{<t}}_n$ for which $\phi^{x_{<t}}_n > \alpha$.
\begin{multline}
    \E_\mu \sum_{t=0}^{N-1} \sum_{i : \phi^{x_{<t}}_n > \alpha} \sum_{x \in \mathcal{X}} \left[\nu^{x_{<t}}_n(x\mid x_{<t}) - \mu(x\mid x_{<t})\right]^2 \leq 2 \alpha^{-2} \E_\mu \sum_{t=0}^{N-1} \sum_{i : \phi^{x_{<t}}_n > \alpha} \sum_{x \in \mathcal{X}}
    \\
    \left[\rhosn(x \mid x_{<t}) - \mu(x\mid x_{<t})\right]^2 +
    \left[ \rhosnm(x \mid x_{<t}) - \mu(x\mid x_{<t}) \right]^2
\end{multline}
Now we note that $\{n : \phi^{x_{<t}}_n > \alpha \} \subset \{n : n < \alpha^{-1}\}$, since $w(\nu^{x_{<t}}_n \mid x_{<t}) \leq w(\nu^{x_{<t}}_j \mid x_{<t})$ for $i > j$. Thus,
\begin{align*}
  &\E_\mu \sum_{t=0}^{N-1} \sum_{n : \phi^{x_{<t}}_n > \alpha} \sum_{x \in \mathcal{X}} \left[\nu^{x_{<t}}_n(x\mid x_{<t}) - \mu(x\mid x_{<t})\right]^2
    \\
    \leq &2 \alpha^{-2} \E_\mu \sum_{t=0}^{N-1} \sum_{i : i < \alpha^{-1}} \sum_{x \in \mathcal{X}} \left[\rhosn(x \mid x_{<t}) - \mu(x\mid x_{<t})\right]^2 +
    \left[ \rhosnm(x \mid x_{<t}) - \mu(x\mid x_{<t}) \right]^2
    \\
    = &2 \alpha^{-2} \sum_{i : i < \alpha^{-1}} \E_\mu \sum_{t=0}^{N-1} \sum_{x \in \mathcal{X}} \left[\rhosn(x \mid x_{<t}) - \mu(x\mid x_{<t})\right]^2 +
    \left[ \rhosnm(x \mid x_{<t}) - \mu(x\mid x_{<t}) \right]^2
    \\
    \leq &2 \alpha^{-2} \sum_{i : i < \alpha^{-1}} 2(6 w(\mu)^{-1} + 3) \leq \alpha^{-3}(24 w(\mu)^{-1} + 12)
    \tagaligneq
\end{align*}

Considering only a subset of these conditional-probability-errors,
\begin{multline}
    \E_\mu \sum_{t = 0}^{N-1} \sum_{x \in \mathcal{X}} \left[\mu(x \mid x_{<t}) - \min_{n : \phi^{x_{<t}}_n > \alpha} \nu_n^{x_{<t}}(x \mid x_{<t})\right]^2 \leq \\
    \E_\mu \sum_{t=0}^{N-1} \sum_{n : \phi^{x_{<t}}_n > \alpha} \sum_{x \in \mathcal{X}} \left[\nu^{x_{<t}}_n(x\mid x_{<t}) - \mu(x\mid x_{<t})\right]^2 \leq \alpha^{-3}(24 w(\mu)^{-1} + 12)
\end{multline}

This completes the proof of (i). Finally, with $\mathbb{U}$ being the uniform distribution,
\begin{align*}
    &\E_\mu \sum_{t = 0}^{N-1} \left[1 - \sum_{x \in \mathcal{X}} \min_{n: \phi^{x_{<t}}_n > \alpha} \nu_n^{x_{<t}}(x \mid x_{<t})\right]^2
    \\
    = &\E_\mu \sum_{t = 0}^{N-1} \left[\sum_{x \in \mathcal{X}} \mu(x \mid x_{<t}) - \min_{n: \phi^{x_{<t}}_n > \alpha} \nu_n^{x_{<t}}(x \mid x_{<t})\right]^2
    \\
    = &\E_\mu \sum_{t = 0}^{N-1} \left[|\mathcal{X}| \E_{x \sim \mathbb{U}(\mathcal{X})} \mu(x \mid x_{<t}) - \min_{n: \phi^{x_{<t}}_n > \alpha} \nu_n^{x_{<t}}(x \mid x_{<t})\right]^2
    \\
    \lequal^{(a)} &|\mathcal{X}|^2 \E_\mu \sum_{t = 0}^{N-1} \E_{x \sim \mathbb{U}(\mathcal{X})} \left[\mu(x \mid x_{<t}) - \min_{n: \phi^{x_{<t}}_n > \alpha} \nu_n^{x_{<t}}(x \mid x_{<t})\right]^2
    \\
    = &|\mathcal{X}| \E_\mu \sum_{t = 0}^{N-1} \sum_{x \in \mathcal{X}} \left[\mu(x \mid x_{<t}) - \min_{n: \phi^{x_{<t}}_n > \alpha} \nu_n^{x_{<t}}(x \mid x_{<t})\right]^2
    \\
    \lequal^{(b)} & |\mathcal{X}| \alpha^{-3}(24 w(\mu)^{-1} + 12)
    \tagaligneq
\end{align*}
where $(a)$ follows from Jensen's Inequality, and $(b)$ follows from Theorem \ref{thm:prederror} (i), which completes the proof of (ii).
\end{proof}

\thmtruthintop*
\begin{proof}
Since, $w(\pi^d \mid h_{<t}) > \alpha \implies \pi^d \in \pa$, we show $\ptrue(\forall t: w(\pi^d \mid h_{<t}) > \alpha) \geq 1 - \alpha w(\pi^d)^{-1}$. First we show that $z_t = w(\pi^d \mid h_{<t})^{-1}$ is a non-negative $\ptrue$-supermartingale.

First, suppose $q_{t+1} = 0$. In this case, $z_{t+1} = z_{t}$, because the posterior weight is only updated when the demonstrator picks an action. Now suppose $q_{t+1} = 1$.
\begin{align*}
    \Etrue[z_{t+1} \mid h_{<t}1] 
    &\equal^{(a)} \sum_{a_t \in \A : \pi^d(a_t \mid h_{<t}) > 0}\pi^d(a_t \mid h_{<t}) w(\pi^d \mid h_{<t}1a_t)^{-1}
    \\
    &\equal^{(b)} \sum_{a_t \in \A : \pi^d(a_t \mid h_{<t}) > 0} \pi^d(a_t \mid h_{<t}) \frac{\sum_{\pi \in \Pi} w(\pi \mid h_{<t})\pi(a_t \mid h_{<t}) }{w(\pi^d \mid h_{<t})\pi^d(a_t \mid h_{<t})}
    \\
    &\lequal^{(c)} \sum_{a_t \in \A} \frac{\sum_{\pi \in \Pi} w(\pi \mid h_{<t})\pi(a_t \mid h_{<t}) }{w(\pi^d \mid h_{<t})}
    \\
    &= z_t \sum_{\pi \in \Pi} w(\pi \mid h_{<t}) \sum_{a_t \in \A} \pi(a_t \mid h_{<t}) = z_t
\end{align*}
where $(a)$ follows because $a_t \sim \pi^d$ when $q_t = 1$, $(b)$ follows from Bayes' rule---the formula for posterior updating, and $(c)$ follows from cancelling, and adding non-negative terms to the sum.

Since $w(\pi^d \mid h_{<t})^{-1}$ is a non-negative supermartingale, by the supermartingale convergence theorem \citep[Thm. 5.4.2]{durrett2010probability},
\begin{equation}
    \ptrue(\exists t: w(\pi^d \mid h_{<t})^{-1} \geq \alpha^{-1}) \leq \alpha w(\pi^d)^{-1}
\end{equation}
so
\begin{equation}
    \ptrue(\forall t: w(\pi^d \mid h_{<t}) > \alpha) \geq 1 - \alpha w(\pi^d)^{-1}
\end{equation}
which implies
\begin{equation} \label{eqn:posteriortruth}
    \ptrue(\forall t: \pi^d \in \pa) \geq 1 - \alpha w(\pi^d)^{-1}
\end{equation}
\end{proof}

\thmimtodem*
\begin{proof}
Recall $\pi^i_\alpha(0, a \mid h_{<t}) = \min_{\pi \in \pa} \pi(1, a \mid h_{<t})$, so if $\pi^d \in \pa$, then $\pi^i_\alpha(0, a \mid h_{<t}) \leq \pi^d(1, a \mid h_{<t})$. Thus, in that case,
\begin{multline}
    \sum_{a \in \A} \va \pi^i_\alpha(0, a \mid h_{<t}) - \pi^d(1, a \mid h_{<t}) \va = \sum_{a \in \A} \pi^d(1, a \mid h_{<t}) - \pi^i_\alpha(0, a \mid h_{<t}) \leq \\
    1 - \sum_{a \in \A} \pi^i_\alpha(0, a \mid h_{<t}) = \theta_q(h_{<t})
\end{multline}
The rest follows easily:
\begin{align*}
    &\Etrue \left[\sum_{t=0}^\infty \left(\sum_{a \in \A} \va \pi^i_\alpha(0, a \mid h_{<t}) - \pi^d(1, a \mid h_{<t}) \va \right)^3 \vd \forall t : \pi^d \in \pa \right]
    \\
    \leq &\Etrue \left[\sum_{t=0}^\infty \theta_q(h_{<t})^3 \vd \forall t : \pi^d \in \pa \right]
    \\
    \lequal^{(a)} &\Etrue \left[\sum_{t=0}^\infty \theta_q(h_{<t})^3 \right] \biggm/ \ptrue(\forall t : \pi^d \in \pa)
    \\
    \lequal^{(b)} &\frac{|\mathcal{A}| \alpha^{-3}(24  w(\pi^d)^{-1} + 12)}{1 - \alpha w(\pi^d)^{-1}}
    \tagaligneq
\end{align*}
where $(a)$ follows because $\theta_q$ is non-negative, and $(b)$ follows from Equation \ref{eqn:posteriortruth} and Theorem \ref{thm:queryprob} (as long as $\alpha < w(\pi^d)$).
\end{proof}

\begin{lemma} \label{lem:deltalesstheta}
For $a \in \A$, let $0 \leq i_a \leq d_a$, and let $\sum_{a \in \A} d_a = 1$. Let $\theta_q = 1 - \sum_{a \in \A} i_a$. Then,
\begin{equation*}
    \Delta := \sum_{a \in \A} (i_a + \theta_q d_a) \log \frac{i_a + \theta_q d_a}{d_a} \leq \theta_q
\end{equation*}
\end{lemma}
\begin{proof}
\begin{multline}
    \sum_{a \in \A} (i_a + \theta_q d_a) \log \frac{i_a + \theta_q d_a}{d_a} = \sum_{a \in \A} (i_a + \theta_q d_a) \log (\frac{i_a}{d_a} + \theta_q) \leq \\
    \left( \sum_{a \in \A} i_a + \theta_q \sum_{a \in \A} d_a \right) \log (1 + \theta_q) = (1 - \theta_q + \theta_q) \log (1 + \theta_q) \leq \theta_q
\end{multline}
\end{proof}

For the remaining proofs, we sometimes consider the restriction of probability measures over $\mathcal{H}^\infty$ to $(\A \times \Ob)^\infty$; that is, we marginalize over the query record. For a history $h_{<t} = q_0a_0 o_0 ... q_{t-1}a_{t-1} o_{t-1}$, let $h^{\setminus}_{<t}$ denote $a_0 o_0 ... a_{t-1} o_{t-1}$. We define the $t$-step KL divergence as follows:
\begin{equation}
    \KL_t(P \mid \mid Q) := \sum_{h^{\setminus}_{<t} \in (\A \times \Ob)^t} P(h^{\setminus}_{<t}) \log \frac{P(h^{\setminus}_{<t})}{Q(h^{\setminus}_{<t})}
\end{equation}

\thmkl*
\begin{proof}
We begin by restricting attention to a particular timestep $t$. Recall $\pi^i_\alpha(0, a \mid h_{<t}) = \min_{\pi' \in \pa} \pi'(1, a \mid h_{<t})$. We abbreviate this quantity $i_a$. We also let $d_a$ denote $\pi^d(1, a \mid h_{<t})$. Note that when $\pi^d \in \pa$,
\begin{equation}
    i_a \leq d_a
\end{equation}
Recall that the query probability $\theta_q = 1 - \sum_{a \in \A} i_a$, and the marginalized probability $\pi^i_\alpha(a \mid h_{<t}) = i_a + \theta_q d_a$. Assuming $h_{<k}$ satisfies $E$, let
\begin{equation}\label{eqn:deltadef}
    \Delta_k := \KL_1 \left(\pi^i_\alpha(\cdot \mid h_{<k}) \va \va \pi^d(\cdot \mid h_{<k}) \right) = \sum_{a \in \A} (i_a + \theta_q d_a) \log \frac{i_a + \theta_q d_a}{d_a}
\end{equation}
By Lemma \ref{lem:deltalesstheta}, $\Delta_k \leq \theta_q$.

Now, we write the $t$-step KL divergence $\KL_t$ as a sum of the expectation of 1-step KL divergences. We'll abbreviate a measure $\p(\cdot \mid E)$ as ${}^E\!\p$.

\begin{align*} \label{ineq:kltotheta}
    \KL_t \left({}^E\!\ptrue \vb \vb {}^E\!\p^{\pi^d}_{\mu} \right) &= \E_{h_{<t} \sim {}^E\!\ptrue} \log \frac{{}^E\!\ptrue(h^{\setminus}_{<t})}{{}^E\!\p^{\pi^d}_{\mu}(h^{\setminus}_{<t})}
    \\
    &\lequal^{(a)} \E_{h_{<t} \sim {}^E\!\ptrue} \log \frac{\ptrue(h^{\setminus}_{<t})/\ptrue(E)}{\p^{\pi^d}_{\mu}(h^{\setminus}_{<t}) / \p^{\pi^d}_{\mu}(E)}
    \\
    &= \E_{h_{<t} \sim {}^E\!\ptrue} \log \frac{\ptrue(h^{\setminus}_{<t})}{\p^{\pi^d}_{\mu}(h^{\setminus}_{<t})} + \log \frac{\p^{\pi^d}_{\mu}(E)}{\ptrue(E)}
    \\
    &\leq \E_{h_{<t} \sim {}^E\!\ptrue} \log \frac{\ptrue(h^{\setminus}_{<t})}{\p^{\pi^d}_{\mu}(h^{\setminus}_{<t})} - \log \ptrue(E)
    \\
    &=: \E_{h_{<t} \sim {}^E\!\ptrue} \log \frac{\ptrue(h^{\setminus}_{<t})}{\p^{\pi^d}_{\mu}(h^{\setminus}_{<t})} + C_\alpha
    \\
    &\equal^{(b)} C_\alpha + \E_{h_{<t} \sim {}^E\!\ptrue} \sum_{k=0}^{t-1} \log \frac{\ptrue(h^{\setminus}_{k} \mid h_{<k})}{\p^{\pi^d}_{\mu}(h^{\setminus}_k \mid h_{<k})}
    \\
    &= C_\alpha + \sum_{k=0}^{t-1} \E_{h_{<k} \sim {}^E\!\ptrue} \E_{h_k \sim {}^E\!\ptrue(\cdot \mid h_{<k})} \log \frac{\ptrue(h^{\setminus}_{k} \mid h_{<k})}{\p^{\pi^d}_{\mu}(h^{\setminus}_k \mid h_{<k})}
    \\
    &= C_\alpha + \sum_{k=0}^{t-1} \E_{h_{<k} \sim {}^E\!\ptrue} \sum_{h^{\setminus}_k \in \A \times \Ob} {}^E\!\ptrue(h^{\setminus}_k \mid h_{<k}) \log \frac{\ptrue(h^{\setminus}_{k} \mid h_{<k})}{\p^{\pi^d}_{\mu}(h^{\setminus}_k \mid h_{<k})}
    \\
    &\leq C_\alpha + \sum_{k=0}^{t-1} \E_{h_{<k} \sim {}^E\!\ptrue} \sum_{h^{\setminus}_k \in \A \times \Ob} \frac{\ptrue(h^{\setminus}_k \mid h_{<k})}{\ptrue(E)} \log \frac{\ptrue(h^{\setminus}_{k} \mid h_{<k})}{\p^{\pi^d}_{\mu}(h^{\setminus}_k \mid h_{<k})}
    \\
    &= C_\alpha + \sum_{k=0}^{t-1} \E_{h_{<k} \sim {}^E\!\ptrue} \frac{1}{\ptrue(E)} \KL_1 \left(\ptrue(\cdot \mid h_{<k}) \vb \vb  \p^{\pi^d}_{\mu}(\cdot \mid h_{<k})\right)
    \\
    &= C_\alpha + \frac{1}{\ptrue(E)} \sum_{k=0}^{t-1} \E_{h_{<k} \sim {}^E\!\ptrue} \KL_1 \left(\pi^i_\alpha(\cdot \mid h_{<k}) \va\va  \pi^d(\cdot \mid h_{<k})\right)
    \\
    &\lequal^{(c)} -\log \ptrue(E) + \frac{1}{\ptrue(E)} {}^E\!\Etrue \sum_{k=0}^{t-1} \theta_q(h_{<k})
    \\
    &\leq -\log \ptrue(E) + \frac{1}{\ptrue(E)^2} \Etrue \sum_{k=0}^{t-1} \theta_q(h_{<k})
    \tagaligneq
\end{align*}
where $(a)$ follows from $h_{<t}$ satisfying $E$ with ${}^E\!\ptrue$-prob. 1, $(b)$ follows because $\mu$ and $\pi^d$ are fair, and $(c)$ follows from Equation \ref{eqn:deltadef} and Lemma \ref{lem:deltalesstheta}.

Finally,
\begin{align*} \label{ineq:jensencube}
    \Etrue \sum_{k=0}^{t-1} \theta_q(h_{<k}) &= t \E_{k \sim \mathbb{U}([t])} \Etrue \theta_q(h_{<k})
    \\
    &= t \left( \left( \E_{k \sim \mathbb{U}([t])} \Etrue \theta_q(h_{<k}) \right)^{3} \right)^{1/3}
    \\
    &\lequal^{(a)} t \left(\E_{k \sim \mathbb{U}([t])} \Etrue \theta_q(h_{<k})^3 \right)^{1/3}
    \\
    &= t \left(\frac{1}{t} \sum_{k=0}^{t-1} \Etrue \theta_q(h_{<k})^3 \right)^{1/3}
    \\
    &\lequal^{(b)} t^{2/3} |\mathcal{A}|^{1/3} \alpha^{-1}(24  w(\pi^d)^{-1} + 12)^{1/3}
    \tagaligneq
\end{align*}
where $(a)$ follows from Jensen's Inequality, and $(b)$ follows from Theorem \ref{thm:queryprob}. Combining this with Inequality \ref{ineq:kltotheta}, and recalling $\ptrue(E) \geq 1-\alpha/w(\pi^d)$, we have
\begin{equation}
    \KL_t \left({}^E\!\ptrue \vb \vb {}^E\!\p^{\pi^d}_{\mu} \right) \leq \frac{\alpha^{-1} |\mathcal{A}|^{1/3} (24  w(\pi^d)^{-1} + 12)^{1/3}}{(1-\alpha/w(\pi^d))^2} t^{2/3} - \log(1-\alpha/w(\pi^d))
\end{equation}
\end{proof}

\end{document}